\def\version{1}

\newif\ifthesis
\thesisfalse
\newif\ifnatbib
\natbibfalse

\ifnum\version=1
  \documentclass[letterpaper,11pt]{article}
  \usepackage[margin=1 in]{geometry}
  \natbibtrue
  \newcommand{\IEEEtitleabstractindextext}[1]{#1}
  \newcommand{\IEEEraisesectionheading}[1]{#1}
  \newcommand{\mykeywords}[1]{\textbf{keywords:} #1}
  \newcommand{\email}[1]{\href{mailto:#1}{#1}}
  \newcommand{\IEEEPARstart}[1]{#1}
\fi
\ifnum\version=2
  \documentclass[10pt,journal,compsoc]{IEEEtran}
  \natbibfalse
  \newcommand{\mykeywords}[1]{\begin{IEEEkeywords}#1\end{IEEEkeywords}}
\fi

\title{The Hidden Convexity of Spectral Clustering\thanks{A short version of this paper previously appeared in the proceedings of the Thirtieth AAAI Conference on Artificial Intelligence \citep{BelkinRV14b}.
    Implementations of the algorithms proposed in this paper can be found at \url{https://github.com/vossj/HBR-Spectral-Clustering}.}}

\ifnum\version=1
\author{ 
  James Voss%
  \footnote{Department of Computer Science and Engineering, the Ohio State University} \\
  \email{vossj@cse.ohio-state.edu}
  \and
  Mikhail Belkin\footnotemark[\value{footnote}] \\
  \email{mbelkin@cse.ohio-state.edu}    
  \and
  Luis Rademacher\footnotemark[\value{footnote}] \\
  \email{lrademac@cse.ohio-state.edu}    
}
\fi

\ifnum\version=2
  \author{
    James Voss,
    Mikhail Belkin,
    and
    Luis Rademacher
    \IEEEcompsocitemizethanks{\IEEEcompsocthanksitem The authors are with the Department of Computer Science, the Ohio State University, Columbus, OH, 43210. \protect\\
      E-mail: (vossj, mbelkin, lrademac)@cse.ohio-state.edu
    }
}
\fi

\usepackage{amsmath, amsfonts, amssymb, amsthm}
\usepackage{algorithm, algorithmicx, algpseudocode}
\usepackage{enumitem}
\usepackage{graphicx}
\usepackage{hyperref}
\ifnatbib
\usepackage[numbers]{natbib}

\else
\usepackage{cite}
\def\citep{\cite}
\fi

\newcommand{\mynewtheorem}[2]{\newtheorem{#1}[thm]{#2}}

\theoremstyle{plain}
\newtheorem{thm}{Theorem}
\mynewtheorem{lem}{Lemma}
\mynewtheorem{lemma}{Lemma}
\mynewtheorem{prop}{Proposition}
\mynewtheorem{fact}{Fact}
\mynewtheorem{cor}{Corollary}
\mynewtheorem{construction}{Construction}
\mynewtheorem{example}{Example}

\theoremstyle{definition}
\mynewtheorem{defn}{Definition}

\theoremstyle{remark}
\mynewtheorem{assump}{Assumption}

\def\final{1}
\ifnum\final=0  
\newcommand{\lnote}[1]{[{\small Luis: \bf #1}]}
\newcommand{\vnote}[1]{[{\small Jimmy: \bf #1}]}
\newcommand{\mnote}[1]{[{\small Misha: \bf #1}]}
\newcommand{\anonnote}[1]{[{\small anon: \bf #1}]}
\newcommand{\sidecomment}[1]{\marginpar{\tiny #1}}
\newcommand{\details}[1]{[[#1]]}
\else 
\newcommand{\lnote}[1]{}
\newcommand{\vnote}[1]{}
\newcommand{\mnote}[1]{}
\newcommand{\anonnote}[1]{}
\newcommand{\sidecomment}[1]{}
\newcommand{\details}[1]{}
\fi  

\usepackage{amsmath}

\newcommand{\abs}[1]{{\ensuremath | #1 |}}
\newcommand{\Abs}[1]{\ensuremath \left| #1 \right|}

\newcommand{\norm}[2][]{\ensuremath{\lVert #2 \rVert_{#1}}}

\newcommand{\sym}{\ensuremath{\mathrm{sym}}}
\newcommand{\rw}{\ensuremath{\mathrm{rw}}}

\newcommand{\R}{\mathbb{R}}

\newcommand{\LL}{\mathcal{L}}

\newcommand{\NN}{\mathcal{N}}
\renewcommand{\SS}{\mathcal{S}}

\DeclareMathOperator{\argmax}{arg\ max}

\DeclareMathOperator{\conv}{conv}
\DeclareMathOperator{\CUT}{Cut}
\DeclareMathOperator{\diag}{diag}
\DeclareMathOperator{\NCUT}{NCut}

\DeclareMathOperator{\RCUT}{RCut}
\DeclareMathOperator{\sign}{sign}
\DeclareMathOperator{\spn}{span}

\newcommand{\grad}{\nabla}

\renewcommand{\vec}[1]{\ensuremath{\mathbf{#1}}}

\newcommand{\suchthat}{\mathrel{|}}

\renewcommand{\norm}[2][]{\ensuremath{\lVert #2 \rVert\ifthenelse{\equal{#1}{}}{}{_{#1}}}}
\newcommand{\norms}[2][]{\ensuremath{\lVert #2 \rVert^2\ifthenelse{\equal{#1}{}}{}{_{#1}}}}

\newcommand*{\Cdot}[1][1.25]{%
  \mathpalette{\CdotAux{#1}}\bullet%
}
\newdimen\CdotAxis
\newcommand*{\CdotAux}[3]{%
  {%
    \settoheight\CdotAxis{$#2\vcenter{}$}%
    \sbox0{%
      \raisebox\CdotAxis{%
        \scalebox{#1}{%
          \raisebox{-\CdotAxis}{%
            $\mathsurround=0pt #2#3$%
          }%
        }%
      }%
    }%
    \dp0=0pt %
    \sbox2{$#2\bullet$}%
    \ifdim\ht2<\ht0 %
      \ht0=\ht2 %
    \fi
    \sbox2{$\mathsurround=0pt #2#3$}%
    \hbox to \wd2{\hss\usebox{0}\hss}%
  }%
}

\def\hbe{\ensuremath{z}}

\def\hbev{\ensuremath{\vec \hbe}}
\def\adim{d}

\def\sphere{\mathbb S}
\def\orthant{Q}
\def\porth{\orthant_+}
\def\Id{\mathcal I}

\newcommand{\vecdot}{\Cdot[0.5]}
\newcommand{\ipCanonical}[2]{#1 \mathbin{\vecdot} #2}
\newcommand{\ipCanonicalp}[2]{(\ipCanonical{#1}{#2})}
\newcommand{\fg}[1][]{F_{g_{#1}}}

\DeclareMathOperator{\dimop}{dim}
\newcommand{\mysubsection}[1]{\subsection{#1}}
\newcommand{\sfrac}[2]{#1 / #2}                

\newcommand{\calS}{\mathcal{S}}
\newcommand{\One}{\mathbf 1}

\newcommand{\edim}{k}
\newcommand{\dn}{\ensuremath{n}}

\newcommand{\SCHAR}{m}
\newcommand{\slope}[2]{\ensuremath{\SCHAR_{#1}^{#2}}}
\newcommand{\LEFT}{\ell}
\newcommand{\RIGHT}{r}
\newcommand{\lslope}[1][]{\slope{#1}{\LEFT}}
\newcommand{\rslope}[1][]{\slope{#1}{\RIGHT}}
\def\myZ{z}
\def\myZv{\vec \myZ}
\def\myD{\mathfrak D}
\def\myd{\mathfrak d}

\def\Zpts{\edim}

\def\findopt{\textsc{HBRopt}}
\def\findenum{\textsc{HBRenum}}

\def\tabs{\mathrm{abs}}
\def\tgau{\mathrm{gau}}
\def\tsig{\mathrm{sig}}
\def\tht{\mathrm{ht}}

\newcommand{\myfootnote}[1]{\footnote{#1}}
\newenvironment{myalgorithm}[1][]{\begin{algorithm}[#1]}{\end{algorithm}}
\newcommand{\mycaption}[2][]{%
  \ifthenelse{\equal{#1}{}}%
  {\caption[#2]{#2}} 
  {\caption[#1]{#2}}} 

\ifnum\version=2
\renewenvironment{proof}[1][Proof]{\begin{IEEEproof}[#1]}{\end{IEEEproof}}
\fi

\begin{document}
\ifnum\version=2
\IEEEtitleabstractindextext{%
\begin{abstract}
  In recent years, spectral clustering has become a standard method for data
  analysis used in a broad range of applications.  
  In this paper we propose a new class of algorithms for multiway spectral
  clustering based on optimization of a certain ``contrast function'' over the
  unit sphere.
  These algorithms, partly inspired by certain Independent Component Analysis
  techniques, are simple, easy to implement and efficient.       

  Geometrically, the proposed algorithms can be interpreted as hidden basis
  recovery by means of function optimization.
  We give a complete characterization of the contrast functions admissible for
  provable basis recovery. We show how these conditions can be interpreted as
  a ``hidden convexity'' of our optimization problem on the sphere;
  interestingly,
  we use efficient convex maximization rather than the more common convex
  minimization.
  We also show encouraging experimental results on real and simulated data. 
\end{abstract}

\mykeywords{spectral clustering, convex maximization, basis recovery}


}
\fi


\maketitle
\ifnum\version=1
\begin{abstract}
  In recent years, spectral clustering has become a standard method for data
  analysis used in a broad range of applications.  
  In this paper we propose a new class of algorithms for multiway spectral
  clustering based on optimization of a certain ``contrast function'' over the
  unit sphere.
  These algorithms, partly inspired by certain Independent Component Analysis
  techniques, are simple, easy to implement and efficient.       

  Geometrically, the proposed algorithms can be interpreted as hidden basis
  recovery by means of function optimization.
  We give a complete characterization of the contrast functions admissible for
  provable basis recovery. We show how these conditions can be interpreted as
  a ``hidden convexity'' of our optimization problem on the sphere;
  interestingly,
  we use efficient convex maximization rather than the more common convex
  minimization.
  We also show encouraging experimental results on real and simulated data. 
\end{abstract}

\mykeywords{spectral clustering, convex maximization, basis recovery}


\fi

\IEEEraisesectionheading{\section{Introduction}\label{sec:introduction}}

\IEEEPARstart{P}{artitioning} a dataset into classes based on a similarity between data points, known as  cluster analysis, 
is one of the most basic and practically important problems in  data analysis and machine learning. It has a vast array of applications from speech recognition to image analysis to bioinformatics and to data compression.
There is an extensive   literature on the subject, including a number of different methodologies as well as their various practical and theoretical aspects~\citep{Jain88}. 

In recent years spectral clustering---a class of methods based on the eigenvectors of a certain matrix, typically the graph Laplacian constructed from data---has become a widely used method for cluster analysis.   
This is due to the simplicity of the algorithm, a number of desirable properties it exhibits and its amenability to theoretical analysis.  
In its simplest form, spectral bi-partitioning is an attractively straightforward algorithm based on  thresholding the second bottom eigenvector of the Laplacian matrix of a graph.
However, the more practically significant problem of multiway spectral clustering\index{spectral clustering!multiway} is considerably more complex. 
While hierarchical methods based on a sequence of binary splits have been used, the most common  approaches use $k$-means or weighted $k$-means clustering in the spectral space or related iterative procedures~\citep{ShiMal00,ng2002spectral,DBLP:journals/jmlr/BachJ06,yu2003multiclass}.
Typical algorithms for multiway spectral clustering follow a two-step process:

\begin{enumerate}
\item
\emph{Spectral embedding:}\index{spectral embedding} A similarity graph for the data is constructed based on the data's feature representation.  If one is looking for $k$ clusters, one constructs the embedding using the bottom $k$ eigenvectors of the graph Laplacian (normalized or unnormalized) corresponding to that graph. 

\item
\emph{Clustering:}\index{spectral clustering!clustering step} In the second step, the embedded data (sometimes rescaled)  is clustered, typically using the conventional/spherical $k$-means algorithms or their variations.
\end{enumerate}

In the first step, the spectral embedding given by the eigenvectors of Laplacian matrices has a number of interpretations.  The meaning can be explained by spectral graph theory as relaxations of multiway cut problems \citep{von2007tutorial}.
In the extreme case of a similarity graph having $k$ connected components, the embedded vectors reside in $\R^k$, and vectors corresponding to the same connected component are mapped to a single point.
There are also  connections to other areas of machine learning and mathematics, in particular to the geometry of the underlying space from which the data is sampled~\citep{BN03}. 


We propose a new class of algorithms for the second step of  multiway spectral clustering.  The starting point is that when the $k$ clusters are perfectly separate, the spectral embedding using the bottom $k$ eigenvectors has a particularly simple geometric form.
For the unnormalized (or  asymmetric normalized) Laplacian, it is simply a (weighted) orthogonal basis in $k$-dimensional space, and recovering the basis vectors is sufficient for cluster identification.
This view of spectral clustering as basis recovery is related to previous observations that the spectral embedding generates a discrete weighted simplex\index{discrete simplex} (see \citep{weber2004perron,ravindran13icml} for some applications).
For the symmetric normalized Laplacian, the structure is slightly more complex, but is still suitable for our analysis.
Moreover, our proposed algorithms can be used without modification.

The proposed approach relies on an optimization problem resembling  certain Independent Component Analysis\index{independent component analysis} techniques, such as FastICA\index{FastICA} (see~\citep{hyvarinen2004independent} for a broad overview).
Specifically,  the problem of identifying $k$ clusters reduces to maximizing  a certain ``admissible" contrast function over a $(k-1)$-sphere.
Our main theoretical contribution is to formulate a general version of the basis recovery problem arising in spectral clustering, and to characterize the set of admissible contrast functions for guaranteed recovery%
\myfootnote{Interestingly, there are
  no analogous recovery guarantees in the ICA setting except for the
  special case of cumulant functions as contrasts.
  In particular, typical versions of FastICA\index{FastICA} are known to have
  spurious maxima~\citep{wei2015study}.}
(Section~\ref{sec:cond-defl-basis}).
Rather than the more usual convex minimization, our analysis is based on  \emph{convex maximization}\index{convex maximization} over a (hidden) convex domain. 
Interestingly, while {\it maximizing} a convex function over a convex domain is generally  difficult (even maximizing a positive definite quadratic form over the continuous cube $[0,1]^n$ is NP-hard\myfootnote{This follows from~\citep{gritzmann19890} together with Fact~\ref{ch-opt:fact:maximum_principle} below.}), our setting allows for efficient optimization.

Based on this theoretical connection between clusters and local maxima of contrast functions over the sphere, we
propose practical algorithms for cluster recovery through function maximization. We discuss the choice of contrast functions  and  provide running time analysis. We also provide a number of encouraging experimental results on synthetic and real-world data sets. 


We also note connections to recent work on geometric recovery. 
\ifnatbib \citet{luis13} \else Anderson et al.\@ \cite{luis13} \fi use the method of moments to recover a continuous
simplex given samples from the uniform probability distribution. 
Like in our work, \ifnatbib \citeauthor{luis13} \else Anderson et al.\@ \fi use efficient enumeration of
local maxima of a function over the sphere.
\ifthesis
\else
Also, one of the results of \ifnatbib \citet{hsu2013learning} \else Hsu and Kakade \cite{hsu2013learning} \fi shows recovery of
parameters in a Gaussian Mixture Model using the moments of order
three, and this result can be thought of as a case of the basis
recovery problem.
\fi

The paper is structured as follows:  In Section~\ref{sec:cond-defl-basis}, we provide our main technical results on basis recovery and briefly outline its connection to spectral clustering.
In Sections~\ref{sec:section3} and~\ref{sec:null_space_L} we introduce spectral clustering and formulate it in terms of basis learning.
In Section~\ref{sec:arb_functions} we provide the main theoretical results for basis recovery in the spectral clustering setting, and discuss algorithic implementation details.
Our experimental results are given in Section~\ref{sec:experiments}.
Finally in Section~\ref{app:Lsym}, we handle the deferred proof details and discuss the admissibility of normalized graph Laplacians for our framework.



\section{Basis Recovery and Spectral Clustering}
\label{sec:cond-defl-basis}

In this section, we provide our main technical results on hidden basis recovery.
Then, we briefly discuss how our results will apply to the spectral clustering setting.

\emph{A Note on Notation.}
In what follows, we will use the following notations.
For a matrix $B$, $b_{ij}$ indicates the element in its $i$\textsuperscript{th} row and $j$\textsuperscript{th} column.
The $i$\textsuperscript{th} row vector of $B$ is denoted $b_{i\vecdot}$, and the $j$\textsuperscript{th} column vector of $B$ is denoted $b_{\vecdot j}$.  
For a vector $\vec v$, $\norm{\vec v}$ denotes its standard Euclidean 2-norm.
Given two vectors $\vec u$ and $\vec v$, $\ipCanonical {\vec u} {\vec v}$ denotes their dot product.
We denote the set $\{1, 2, \dotsc, k\}$ by $[k]$.
We denote by $\One_{\calS}$ the
indicator vector for the set $\calS$, i.e.\@ the
vector which is $1$ for indices in $\calS$ and $0$ otherwise.  The null space of a matrix $M$ is denoted $\NN(M)$.
We denote the unit sphere in $\R^d$ by $\sphere^{d-1}$.
For points $p_1, \dotsc, p_m$, $\conv(p_1, \dotsc, p_m)$ will denote their convex hull.
All angles are given in radians, and $\angle(\vec u, \vec v)$ denotes the angle between the vectors $\vec u$ and $\vec v$ in the domain $[0, \pi]$.
We use $\mapsto$ to define anonymous functions; for instance $t \mapsto t^2$ is the function $f : \R \rightarrow \R$ defined by $f(t) := t^2$.
Finally, for $\mathcal{X}$ a subspace of $\R^{\adim}$, $P_{\mathcal{X}}$ denotes the square matrix  corresponding to the orthogonal projection from $\R^{\adim}$ to $\mathcal X$.

\subsection{Basis Recovery via Convex Maximization}
The main technical results of this section deal with reconstructing a hidden basis by simple optimization techniques.
For this purpose, we introduce the following class of functions.
\begin{defn}\label{ch-opt:def:OBEF}
  A function $F : \R^\adim \rightarrow \R$ is said to be an \emph{orthogonal basis encoding function}\index{Basis Encoding Function!orthogonal} (orthogonal BEF) if there exists an orthonormal basis $\hbev_1, \dotsc, \hbev_\adim$ of $\R^\adim$ and functions $g_i : \R \rightarrow \R$ such that $F(\vec u) = \sum_{i=1}^\adim g_i(\ipCanonical{\vec u}{\hbev_i})$.
\end{defn}

We will assume throughout that the functions $g_i$ (and hence $F$) are
continuously differentiable.
In this section, we provide conditions under which recovery of the hidden
basis $\hbev_1, \dotsc, \hbev_\adim$ (up to sign) can be guaranteed for an
orthogonal BEF using simple function maximization techniques.
To motivate our conditions, it will be useful to first consider a classic
problem which fits into the orthogonal BEF framework: the eigendecomposition of positive definite symmetric matrices.

\begin{example}[Symmetric PSD Matrix Eigendecompositions]\label{ch-opt:ex:Eigendecomp-Degenerate}
  Let $A$ be a symmetric positive semi-definite matrix with
  eigendecomposition $A = \sum_{i=1}^\adim \lambda_i \hbev_i \hbev_i^T$.
  The function $F_A : \R^\adim \rightarrow \R$ defined by $F_A(\vec u) := \vec u^T A \vec u = \sum_{i = 1}^\adim \lambda_i \ipCanonicalp{\vec u}{\hbev_i}^2$ is an orthogonal BEF with the functions $g_i : \R \rightarrow \R$ defined as $g_i(x) := \lambda_i x^2$.
  If the eigenvalues are ordered such that $\lambda_1 > \lambda_2 > \dotsc > \lambda_\adim$,
  then the directions $\pm \hbev_1$ are the maxima (local and global) of $F_A$ on the domain $\sphere^{\adim - 1}$.\vnote{Try to find a citation for this preceding fact}
  Further, after $\pm \hbev_1$ is recovered, we may maximize $F_A$ in the orthogonal complement of $\hbev_1$ to recover $\pm \hbev_2$.
  This deflationary procedure can be extended to recover all eigenvectors of $F_A$ (see Algorithm~\ref{alg:deflation} for the idea).
  
  However, when $A$ has repeated eigenvalues, then its eigendecomposition is no longer uniquely defined.
  For the identity matrix $\Id$, any orthonormal basis in $\R^\adim$ can be used to form its eigenvectors, and the function $F_\Id(\vec u) = 1$ for any choice of $\vec u \in \sphere^{\adim - 1}$.
  In general, the hidden basis recovery problem arising in the
  eigendecomposition problem is only uniquely defined when there are
  no repeat eigenvalues.
\end{example}

As pointed out by the Example~\ref{ch-opt:ex:Eigendecomp-Degenerate}, 
we will need to understand the conditions under which a deflationary approach to maximizing a BEF $F$ on $\sphere^{\adim - 1}$ (see Algorithm~\ref{alg:deflation}) can be guaranteed to recover the hidden basis $\hbev_1, \dotsc, \hbev_{\adim}$.
We also wish that the hidden basis $\hbev_1, \dotsc, \hbev_\adim$ be uniquely defined by the BEF $F$.
It turns out that the following assumption is sufficient for performing guaranteed basis recovery.
\begin{assump}[Strict convexity]\label{assump:strict-convexity}\index{Assumption!\ref{assump:strict-convexity}}
  For all $i \in [\adim]$, $t \mapsto g_i(\sign(t) \sqrt{ \abs t })$ is strictly convex\index{strict convexity}.
\end{assump}

More formally, we have the following result.

\begin{myalgorithm}[bt]
  \mycaption[Deflationary scheme for hidden basis
  recovery]{\label{alg:deflation} The deflationary scheme for hidden
    basis recovery.
    This is an abstract algorithm which when given access to an
    orthogonal BEF $F(\vec u) = \sum_{i=1}^\adim g_i(\ipCanonical{\vec
      u}{\hbev_i})$ satisfying
    Assumption~\ref{assump:strict-convexity}, it recovers and returns
    estimates of the hidden basis directions $\hbev_1, \dotsc,
    \hbev_\adim$ up to unknown signs and potentially an unknown
    permutation.}
  \begin{algorithmic}[1]
    \For {$i \leftarrow 1$ to $\adim$}
      \State Find $\tilde \hbev_i$ a local maximizer of $F$ on $\sphere^{\adim - 1} \cap \spn(\{ \tilde \hbev_1, \tilde \hbev_2, \dotsc, \tilde \hbev_{i-1} \})^\perp$
    \EndFor
    \State \Return $\tilde \hbev_1, \dotsc, \tilde \hbev_\adim$.  
  \end{algorithmic}
\end{myalgorithm}

\begin{thm}\label{ch-opt:thm:hbfopt_strict_conv}
  Suppose that $F$ is an orthogonal BEF satisfying Assumption~\ref{assump:strict-convexity}.
  Then, the set of local maxima of $F$ on the unit sphere is non-empty and contained in the set $\{ \pm \hbev_1, \dotsc, \pm \hbev_\adim \}$.
\end{thm}
The Assumption~\ref{assump:strict-convexity} is sufficient for hidden
basis recovery in the sense of the following Corollary.
Its proof is an exercise in induction on the number of recovered
vectors $\tilde \hbev_j$, where the inductive step is a result of
Theorem~\ref{ch-opt:thm:hbfopt_strict_conv}.
\begin{cor}\label{ch-opt:cor:deflation-success}
  If $F$ is an orthogonal BEF satisfying Assumption~\ref{assump:strict-convexity}, then the abstract Algorithm~\ref{alg:deflation} returns vectors $\tilde \hbev_1, \dotsc, \tilde \hbev_\adim$ which recover the directions $\hbev_1, \dotsc, \hbev_\adim$ up to a choice of signs and permutation.
  More precisely, there exists sign $s_i \in \{ \pm 1 \}$ and a permutation $p$ of $[\adim]$ such that $\hbev_i = s_i \tilde \hbev_{p(i)}$ for each $i \in [\adim]$.
\end{cor}
Before proceding with the proof of
Theorem~\ref{ch-opt:thm:hbfopt_strict_conv}, it is worth discussing
the importance of \emph{strict} convexity in
Assumption~\ref{assump:strict-convexity}.
In the case of the matrix eigendecomposition
Example~\ref{ch-opt:ex:Eigendecomp-Degenerate} with the identity
matrix $\Id$, we constructed an orthogonal BEF with contrast functions
$g_i(t) = t^2$ which satisfy that each $g_i(\sign(t) \sqrt{\abs t}) =
t$ is convex but not \emph{strictly} convex.
The function $F_\Id(\vec u)$ is constant on the unit sphere, and there
is no uniquely defined hidden basis (or eigenvector basis) for the
identity matrix.
In this sense, it does not suffice for $t \mapsto g(\sign(t)\sqrt{\abs
  t})$ to be convex.

Interestingly, the only issue which can arise when strict convexity is relaxed to convexity in Assumption~\ref{assump:strict-convexity} is that the function $F$ may plateau (become constant) on regions within the unit sphere $\sphere^{\adim - 1}$.
Strict convexity is one way to ensure that this does not happen.
Nevertheless, the problem of recovering the eigendecomposition of a positive definite symmetric matrix $A$ (Example~\ref{ch-opt:ex:Eigendecomp-Degenerate}) is a limit case of our framework.
Moreover, Algorithm~\ref{alg:deflation} can be used to perform eigenvector recovery since one does not require uniqueness of the eigenvector basis.

The intuition behind Assumption~\ref{assump:strict-convexity} is captured in the proof of Theorem~\ref{ch-opt:thm:hbfopt_strict_conv}.
The main idea is to introduce a change of variable and recast maximization of $F$ over the unit sphere as a \emph{convex maximization}\index{convex maximization} problem defined over a (hidden) convex domain.
\begin{proof}[Proof of Theorem~\ref{ch-opt:thm:hbfopt_strict_conv}]
  We will use the following Fact about convex maximization
  (see~\cite[Chapter 32]{MR1451876} for an overview of concepts
  related to convex maximization).

For a convex set $K$, a point $x \in K$ is said to be an \emph{extreme
  point}\index{extreme point} if $x$ is not equal to a strict convex
combination of two other points in $K$.
\begin{fact}\label{ch-opt:fact:maximum_principle} 
  Suppose that $K$ is a closed and bounded convex set.
  Let $f : K \rightarrow \R$ be a \emph{strictly} convex
  function. 
  Then, the set of local maxima of $f$ on $K$ is non-empty and contained in the set of extreme points of $K$.
\end{fact}

  As $\hbev_1, \dotsc, \hbev_\adim$ form an orthonormal basis of the space, we may simplify notation and work in the coordinate system in which $\hbev_1, \dotsc, \hbev_\adim$ are the canonical vectors $\vec e_1, \dotsc, \vec e_\adim$.
  We define $\Delta^{\adim - 1} := \conv(\vec e_1, \dotsc, \vec e_\adim)$ a (hidden) simplex, and
  $\porth^{\adim - 1} := \{ \vec u \in \sphere^{\adim} \suchthat u_i \geq 0 \text{ for all } i \in [\adim]\}$ the restriction of the sphere onto the positive orthant.
  By the symmetries of the problem, it suffices to show that the set $S$ of local maxima of $F$ with respect to $\porth^{\adim - 1}$ is non-empty and that $S \subset \{ \vec e_1, \dotsc, \vec e_\adim\}$.

  The main idea is to use the change of variable $\psi : \porth^{\adim - 1} \rightarrow \Delta^{\adim - 1}$ defined by $\psi_i(\vec u) = u_i^2$.
  Since
  \begin{equation}\label{ch-opt:eq:hbf-simplex}
    F \circ \psi^{-1}(\vec x) 
    = \sum_{i=1}^\adim g_i(\psi^{-1}_i(\vec x))
    = \sum_{i=1}^\adim g_i(\sqrt{x_i}) \ ,
  \end{equation}
  then by Assumption~\ref{assump:strict-convexity}, $F \circ \psi^{-1} : \Delta^{\adim - 1} \rightarrow \R$ is a strictly convex function defined on a closed and bounded convex domain.
  By Fact~\ref{ch-opt:fact:maximum_principle}, we note that the set $S'$ of local maxima of $F \circ \psi^{-1}$ on $\Delta^{\adim - 1}$ is nonempty and contained in the set $\{\vec e_1, \dotsc, \vec e_\adim\}$ of extreme points of $\Delta^{\adim - 1}$.
  Pulling back to $\porth^{\adim - 1}$, we see that $S = \psi^{-1}(S')$ is a non-empty subset of $\{\vec e_1, \dotsc, \vec e_\adim\}$. 
\end{proof}

\subsection{Spectral Clustering as Basis Recovery}
\label{subsec:spectral-basis-outline}
It turns out that orthogonal basis recovery has direct implications
for spectral clustering.
In particular, when an $\dn$-vertex similarity graph $G$ has $\Zpts$
connected components corresponding to the desired clusters, 
it will be seen in section~\ref{sec:null_space_L} that the spectral embedding into $\R^{\Zpts}$ maps
each vertex $\vec v_i$ in the $j$\textsuperscript{th} connected
component onto a ray protruding from the origin in a direction
$\hbev_j$.
It happens that the directions $\hbev_1, \dotsc, \hbev_{\Zpts}$ are
orthogonal.
We let $\vec x_{i}$ denote the embedded points and we construct the
function
\begin{equation*}
  \fg(\vec u) := \frac 1 \dn \sum_{i=1}^\dn g(\abs{ \ipCanonical {\vec u}{\vec x_{i}} } ) \ ,
\end{equation*}
from the embedded data and the contrast function $g : \R \rightarrow \R$.

To see that $\fg$ is actually an orthogonal BEF, we consider the
following theoretical construction: Let $\SS_1, \dotsc \SS_\edim$ be
the vertex index sets corresponding to the distinct components of the
graph $G$, and define the functions $g_j : \R \rightarrow \R$ for all
$j \in [\edim]$ by $g_j(t) = \frac 1 \dn \sum_{i \in \SS_j} g(t \norm{\vec x_i})$.
Then, it may be verified that $F_g(\vec u) = \sum_{j = 1}^\edim
g_j(\ipCanonical{\vec u}{\hbev_j})$, which takes on the form of an
orthogonal BEF\index{spectral clustering!orthogonal BEF}.
In particular, we will be able to recover the directions $\hbev_1,
\dotsc, \hbev_\edim$ corresponding to the desired clusters by
maximizing the function $F_g$ on the unit sphere $\sphere^{\edim -
  1}$.

Due to the special form of orthogonal BEF which arises in spectral clustering, we will have slightly stronger guarantees.
In particular, it will be seen (Theorem~\ref{thm:complete_enumeration} and Theorem~\ref{thm:spectralembBEF-complete-enumeration-general}) all of the directions of $\pm \hbev_1, \dotsc, \pm \hbev_\edim$ are strict local maximum of $\fg$ on $\sphere^{\edim - 1}$ instead of just some.

\section{Spectral Clustering Problem Statement}\label{sec:section3}


Let $G = (V, A)$ denote a \emph{similarity graph}\index{similarity graph} where $V$ is a set of $\dn$ vertices
and $A$ is an adjacency matrix with non-negative weights.
Two vertices $i, j \in V$ are incident if $a_{ij} > 0$, and the value of $a_{ij}$
is interpreted as a measure of the similarity between the vertices.  In spectral
clustering, the goal is to partition the vertices of a graph into sets
$\calS_1, \dotsc, \calS_\edim$ such that these vertex sets form natural clusters in the graph.
In the most basic setting,
$G$ consists of $\edim$ connected components, and the natural clusters should be the
components themselves.  In this case, if $i' \in \calS_i$ and $j' \in \calS_j$ then
$a_{i'j'} = 0$ whenever $i \neq j$.  For convenience, we can consider the 
vertices of $V$ to be indexed such that all indices in $\calS_i$ precede all 
indices in $\calS_j$ when $i < j$.  
The matrix $A$ takes on the form:
\begin{equation*}
  A = \left( \begin{array}{cccc}
        A_{\calS_1} & 0 & \cdots & 0 \\
        0 & A_{\calS_2} & \cdots & 0 \\
        \vdots & \vdots & \ddots & \vdots \\
        0 & 0 & \cdots & A_{\calS_\edim}
      \end{array} \right) \  ,
\end{equation*}
a block diagonal matrix.  In this setting, spectral clustering can be viewed as a
technique for reorganizing a given similarity matrix $A$ into such a
block diagonal matrix. 


In practice, $G$ rarely consists of $\Zpts$ truly disjoint connected
components.
Instead, one typically observes a matrix $\tilde A = A + E$ where $E$
is a perturbation from the clean setting.
The goal of spectral clustering is to permute the rows and columns of
$\tilde A$ to form a matrix which is nearly block diagonal and to
recover the corresponding clusters.  




\section{The Spectral Embedding}\index{spectral embedding}\label{sec:null_space_L}

Given an $\dn$-vertex similarity graph $G = (V, A)$, let $D$ be the
diagonal degree matrix with non-zero entries $d_{ii} = \sum_{j \in V}
a_{ij}$.  
The graph Laplacian\index{graph Laplacian}
is defined as $L := D - A$.
The following well known property of the graph Laplacian
(see \citep{von2007tutorial} for a review) helps shed light on its
importance: Given $\vec u \in \R^\dn$,
\begin{equation}\label{eq:Laplacian}
  \vec u^T L \vec u = \frac 1 2 \sum_{i, j  \in V} a_{ij}(u_i - u_j)^2 \ .
\end{equation}
The graph Laplacian $L$ is symmetric positive semi-definite as equation~\eqref{eq:Laplacian} cannot be negative. 
Further, $\vec u$ is a 0-eigenvector of $L$ (or equivalently, $\vec u
\in \NN(L)$\index{graph Laplacian!null space})
if and only if $\vec u^TL \vec u = 0$.
When $G$ consists of $\Zpts$ connected components with indices in the sets
$\calS_1, \dotsc, \calS_{\Zpts}$, inspection of equation~\eqref{eq:Laplacian} gives that
$\vec u \in \NN(L)$ precisely when $\vec u$ is piecewise constant on each $\calS_i$.
In particular, 
\begin{equation}\label{ch-opt:eq:Laplacian-null-space}
\{\abs{\calS_1}^{-\sfrac 1 2} \One_{\calS_1}, \dotsc, \abs{\calS_{\Zpts}}^{-\sfrac 1 2} \One_{\calS_{\Zpts}} \}
\end{equation}
is an orthonormal basis for $\NN(L)$.

In general, letting $X \in \R^{\dn \times \Zpts}$ contain an orthogonal basis of $\NN(L)$, it cannot be guaranteed that the rows of $X$ will act as indicators of the various classes, as the columns of $X$ have only been characterized up to a rotation within the subspace $\NN(L)$.
However, the rows of $X$ are contained in a scaled orthogonal basis of $\R^\Zpts$ with the basis directions corresponding to the various classes.
We formulate this result as follows (see~\citep{weber2004perron}, \citep[Proposition 5]{Verma03acomparison}, and \citep[Proposition 1]{ng2002spectral} for related statements).


\begin{prop} \label{prop:discrete-simplex}
  Let the similarity graph $G=(V, A)$ contain $\Zpts$ connected components with indices in the sets $\calS_1, \dotsc, \calS_{\Zpts}$, 
  let $\dn = \abs{V}$, and let 
  $L$ be the graph Laplacian of $G$.
  Then, $\NN(L)$ has dimensionality $\Zpts$.
  Let $X = (x_{\vecdot 1},  \dotsc, x_{\vecdot \Zpts})$ contain $\Zpts$ scaled, 
  orthogonal column vectors forming a basis of $\NN(L)$ such that $\norm{x_{\vecdot j}} = \sqrt \dn$ for each $j \in [\Zpts]$.
  Then, there exist weights $w_1, \dotsc, w_{\Zpts}$ with 
  $w_j = \frac{\abs{\calS_j}}{\dn}$ 
  and
  mutually orthogonal vectors $\myZv_1, \dotsc, \myZv_{\Zpts} \in \R^{\Zpts}$ such that
  whenever $i \in \calS_j$, the row vector $x_{i \vecdot} = \frac 1 {\sqrt{w_j}} \myZv_j^T$.
\end{prop}

\begin{proof}
We define the matrix $M_{\calS_i} := \One_{\calS_i} \One_{\calS_i}^T$.
$P_{\NN(L)}$ can be constructed from any orthonormal basis of $\NN(L)$.
Using the two bases $\{\abs{\calS_1}^{-\sfrac 1 2} \One_{\calS_1}, \dotsc, \abs{\calS_{\Zpts}}^{-\sfrac 1 2} \One_{\calS_{\Zpts}}\}$ and $\{\frac 1 {\sqrt \dn}x_{\vecdot 1}, \dotsc, \allowbreak \frac 1 {\sqrt \dn}x_{\vecdot \Zpts}\}$ yields:
\begin{equation*} 
  P_{\NN(L)} = \sum_{i = 1}^\Zpts \abs{\calS_i}^{-1} M_{\calS_i} 
  \quad \text{and} \quad  P_{\NN(L)} = \frac 1 \dn XX^T \ .
\end{equation*}
Thus for $i, j \in V$, $\frac 1 \dn \ipCanonical{x_{i \vecdot}}{x_{j \vecdot}} = (P_{\NN(L)})_{ij}$.  
In particular, if there exists $\ell \in [\Zpts]$ such that $i, j \in \calS_{\ell}$, then $\frac 1 \dn \ipCanonical{x_{i \vecdot}}{x_{j \vecdot}} = \abs{\calS_{\ell}}^{-1}$.
When $i$ and $j$ belong to separate clusters, then $x_{i \vecdot} \perp x_{j \vecdot}$.

If $i, j \in \SS_j$, then
\[
\cos(\angle(x_{i \vecdot}, x_{j \vecdot})) = \frac{\ipCanonical{x_{i \vecdot}}{x_{j\vecdot}}}{\norm{x_{i\vecdot}}\norm{x_{j\vecdot}}} = \frac{\abs{S_\ell}^{-1}}{\abs{S_\ell}^{-\sfrac 1 2}\abs{S_\ell}^{-\sfrac 1 2}}  = 1 \ , 
\]
implies that $x_{i\vecdot}$ and $x_{j\vecdot}$ are in the same direction.  
As they also have the same magnitude, $x_{i \vecdot}$ and $x_{j \vecdot}$
coincide for any two indices $i$ and $j$ belonging to the same component of $G$.

Thus letting $w_i := \frac {\abs{\calS_i}} \dn$ for $i = 1, \dotsc, \Zpts$, there are $\Zpts$ perpendicular 
vectors $\myZv_1, \dotsc, \myZv_{\Zpts}$ corresponding to the $\Zpts$ connected components of $G$ such that $x_{i \vecdot} = \frac 1 {\sqrt{w_{\ell}}}\myZv_{\ell}^T$ for all
$i \in \calS_{\ell}$.
\end{proof}

Proposition~\ref{prop:discrete-simplex} demonstrates that using the null space of the graph Laplacian\index{graph Laplacian!null space},
the $\Zpts$ connected components in $G$ are mapped to $\Zpts$ scaled, orthogonal basis vectors in $\R^{\Zpts}$.
Of course, under a perturbation of $A$, the interpretation of
Proposition~\ref{prop:discrete-simplex} must change.  In particular, $G$ will no longer consist of $\Zpts$ 
connected components, and instead of using only vectors in $\NN(L)$,
$X$ must be constructed using the eigenvectors corresponding to the lowest $\Zpts$ eigenvalues of $L$. 
With the perturbation of $A$ comes a corresponding 
perturbation of the eigenvectors in $X$.
Using the perturbation theory of symmetric matrices, it can be shown that when the perturbation is not too large, the structure of $X$ is approximately maintained (see \citep{davis1970rotation,von2007tutorial}).

Due to different properties of the resulting spectral embeddings,
normalized graph Laplacians are often used in place of $L$ for
spectral clustering, in particular the symmetric normalized
Laplacian\index{graph Laplacian!symmetric normalized} $L_{\sym} :=
D^{-\sfrac 1 2} L D^{-\sfrac 1 2}$ and the asymmetric normalized
Laplacian\index{graph Laplacian!asymmetric normalized} $L_{\rw} :=
D^{-1} L$.
These normalized Laplacians are often viewed as more stable to
perturbations of the graph structure.
Further, spectral clustering with $L_{\sym}$ has a nice interpretation
as a relaxation of the NP-hard multi-way normalized graph cut problem
\citep{yu2003multiclass}, and the use of $L_{\rw}$ has connections to
the theory of Markov chains \citep{deuflhard2000identification,MeiShi01}.

For simplicity, we focus first on the unnormalized graph Laplacian $L$.
However, when $G$ consists of $\Zpts$ connected components,
$\NN(L_{\rw})$ happens to be identical to $\NN(L)$.
The algorithms which we will
propose for spectral clustering turn out to be equally valid when
using any of $L$, $L_{\sym}$, or $L_{\rw}$, though the structure of
$\NN(L_{\sym})$ gives rise to a slightly more complicated ray-based basis
structure.
The discussion of $\NN(L_{\sym})$ and its admissibility are deferred to Section~\ref{app:Lsym}.

\section{Basis Recovery for Spectral Clustering}
\label{sec:arb_functions}
We now focus on the second step of spectral clustering, which is clustering the points embedded by the Laplacian embedding into the desired clusters.
In particular, we will now demonstrate that the embedded data (the rows of $X$ in Proposition~\ref{prop:discrete-simplex}) may be used to construct a function optimization problem whereby the maxima structure of the function can be used to recover the desired clusters.
\begin{construction}\label{constr:spectral-Fg-L}
  Given a graph $G$ with $\dn$ vertices and $\Zpts$ connected
  components, let $X$; $\calS_1, \dotsc, \calS_{\Zpts}$; $w_1, \dotsc,
  w_{\Zpts}$; $\myZv_1, \dotsc, \myZv_{\Zpts}$; and
  $L$ as in Proposition~\ref{prop:discrete-simplex}.
  We construct a function $\fg:\sphere^{\Zpts-1}\rightarrow \R$ on the
  unit sphere using a contrast function $g:[0,
  \infty)\rightarrow\R$ where it is assumed that $t \mapsto g(\sqrt t)$ is strictly convex. 
  We construct $\fg$ as
  \begin{equation} \label{eq:f_defn}
    \fg(\vec u) := \frac 1 \dn \sum_{i=1}^\dn g(\abs{\ipCanonical{\vec u}{x_{i\vecdot}}}) \ .
  \end{equation}
  Using Proposition~\ref{prop:discrete-simplex}, this may be equivalently written as
  \begin{equation} \label{eq:f_weighted_form}
    \fg(\vec u) = \sum_{i = 1}^\Zpts w_i g(\tfrac 1 {\sqrt{w_i}} \abs{\ipCanonical {\vec u} {\myZv_i}}) \ .  
  \end{equation}
\end{construction}


In Construction~\ref{constr:spectral-Fg-L}, the vectors $\myZv_1,
\dotsc, \myZv_{\Zpts}$ form an unseen orthonormal basis of $\R^{\Zpts}$, and
each weight $w_i = \frac{\abs{\calS_i}} \dn$ is the fraction of the rows
of $X$ indexed as $x_{\ell\vecdot}$ which are embedded from the
$i$\textsuperscript{th} component of $G$ and which coincide with the
point $\frac 1 {\sqrt {w_i}} \myZv_i^T$.
Since each embedded point in the $i$\textsuperscript{th} cluster lies on the
line through $\myZv_i$ and $-\myZv_i$, it suffices to recover the
basis directions $\myZv_1, \dotsc, \myZv_\edim$ up to sign in order to
cluster the points.
Our idea is to show that $\fg$ is an orthogonal BEF\index{spectral clustering!orthogonal BEF} which satisfies Assumption~\ref{assump:strict-convexity} with the directions $\myZv_1, \dotsc, \myZv_\edim$ corresponding to the BEF basis.
As such, we will be able to use the maxima structure of $\fg$ on $\sphere^{\edim - 1}$ in order to recover the hidden basis and thence the desired clustering.


We use equation~\eqref{eq:f_weighted_form} to see that $\fg$ is a special form of orthogonal BEF with the functions $g_i$ (see Definition~\ref{ch-opt:def:OBEF}) defined by $g_i(t) := w_i g(\frac 1 {\sqrt {w_i}} \abs{t})$.
Further, since $t \mapsto g(\sqrt t)$ is strictly convex, we see that $t \mapsto g_i(\sign(t)\sqrt{\abs t})$ is strictly convex for all $i \in [\edim]$, and hence $\fg$ satisfies Assumption~\ref{assump:strict-convexity}.
However, due to the special form of $\fg$, each of the directions $\{\pm \myZv_i : i \in [\Zpts]\}$ are maxima of $\fg$ over $\sphere^{\edim - 1}$ (as opposed to just some, cf.\@ Theorem~\ref{ch-opt:thm:hbfopt_strict_conv}).



\begin{thm} \label{thm:complete_enumeration}
  Let $\fg:\sphere^{\Zpts-1}\rightarrow \R$ and
  $\myZv_1,\dotsc,\myZv_\edim$ be defined as in
  Construction~\ref{constr:spectral-Fg-L}.
  Then, the set $\{\pm \myZv_i : i \in [\Zpts]\}$ is a complete
  enumeration of the local maxima of $\fg$.
\end{thm}

We defer the proof of Theorem~\ref{thm:complete_enumeration} to section~\ref{sec:spectral-contrast-admissibility}.
We also provide and prove the analogous result for when $F_g$ is constructed using the Laplacian embedding arising from $L_{\rw}$ or $L_{\sym}$ in section~\ref{sec:spectral-contrast-admissibility}.

As $\fg$ 
is an orthogonal BEF, it follows from the discussion in section~\ref{sec:cond-defl-basis} that by enumerating the local maxima of $\fg$ using a deflationary scheme, we may recover the hidden basis $\myZv_1, \dotsc, \myZv_\edim$ corresponding to the graph clusters.
By Theorem~\ref{thm:complete_enumeration}, we get slightly more flexibility in our algorithmic design since it is known that each of the directions $\myZv_1, \dotsc, \myZv_\edim$ is a local maximum of $\fg$ on $\sphere^{\edim - 1}$, and therefore we have room to relax the orthogonality constraint from the prototypical deflationary scheme (Algorithm~\ref{alg:deflation}) when designing algorithms for hidden basis recovery in the spectral clustering setting.

\mysubsection{Proposed Algorithms}\label{sec:algorithms}

We now design a new class of algorithms for spectral clustering.  
Given a similarity graph $G = (V, A)$ containing $\dn$ vertices,
define a graph Laplacian $\tilde L$ among $L$, $L_{\rw}$, and $L_{\sym}$ (reader's choice).
Viewing $G$ as a perturbation of a graph consisting of $\Zpts$ connected
components, construct $X \in \R^{\dn \times \Zpts}$ such that $x_{\vecdot i}$ gives the 
eigenvector corresponding to the $i$\textsuperscript{th} smallest eigenvalue
of $\tilde L$ with scaling $\norm{x_{\vecdot i}} = \sqrt \dn$. 

With $X$ in hand, choose a contrast function $g$ satisfying the strict
convexity condition from Assumption~\ref{assump:strict-convexity}.
From $g$, the function $\fg(\vec u) = \frac 1 \dn \sum_{i = 1}^\dn
g(\ipCanonical {\vec u} {x_{i \vecdot}})$ is defined on $\sphere^{\Zpts-1}$
using the rows of $X$.
The local maxima of $\fg$ correspond to the desired clusters of the
graph vertices.
Since $\fg$ is a symmetric function, if $\fg$ has a local maximum at
$\vec u$, $\fg$ also has a local maximum at $-\vec u$.
However, the directions $\vec u$ and $-\vec u$ correspond to the same line
through the origin of $\R^{\Zpts}$ and form an equivalence class, with
each such equivalence class corresponding to a cluster.

Our first goal is to find local maxima of $\fg$ corresponding to distinct equivalence classes.
We will use that the desired maxima of $\fg$ should be approximately orthogonal to each other.
Once we have obtained local maxima $\vec u_1, \dotsc, \vec u_\edim$ of $\fg$, we cluster the vertices of $G$ by placing
vertex $i$ in the $j$\textsuperscript{th} cluster using the rule
$j = \argmax_\ell \abs{\ipCanonical{\vec u_\ell}{x_{i \vecdot}}}$.
We sketch two algorithmic ideas in \findopt\@ and \findenum\@ (where HBR stands for hidden basis recovery).

\begin{myalgorithm}
  \mycaption[\findopt]{
    Finds the local maxima of $\fg$ defined from the embedded vertices
    $x_{i\vecdot}$ which we want to cluster.
    The second input $\eta$ is the learning rate (step size).}%
  \begin{algorithmic}[1]%
  \Function{\findopt}{$X$, $\eta$}
  \State $C \leftarrow \{ \}$
  \For {$i \leftarrow 1$ to $\Zpts$}
  \State Draw $\vec u$ uniformly from $\sphere^{\Zpts-1} \cap \spn(C)^\perp$
      \label{alg1:step:draw_from_sphere}
    \Repeat
       \State\label{alg1:step:ascent} 
$\vec u \leftarrow \vec u + \eta (\grad \fg(\vec u) - \vec u\vec u^T \grad \fg(\vec u) )\newline\phantom.\qquad\qquad\qquad$($=\vec u + \eta P_{\vec u^\perp} \grad \fg(\vec u)$)%
       \State $\vec u \leftarrow P_{\spn(C)^\perp} \vec u$ \hspace{0.75 in} 
       \label{alg1:step:optional} 
       \State $\vec u \leftarrow \frac {\vec u} {\norm{\vec u}}$
    \Until {Convergence}
    \State Let $C \leftarrow C \cup \{ \vec u \}$
  \EndFor
  \State \Return $C$
  \EndFunction
 \end{algorithmic}%
\end{myalgorithm}

\findopt\@ is a form of projected gradient ascent which more fully
implements the deflationary scheme of Algorithm~\ref{alg:deflation}.
The parameter $\eta$ is the learning rate.
Each iteration of the repeat-until loop moves $\vec u$ in the
direction of steepest ascent.
For gradient ascent in $\R^{\Zpts}$, one would expect
step~\ref{alg1:step:ascent} 
of \findopt\@ to read $\vec u \leftarrow \vec u + \eta \grad \fg(\vec
u)$.
However, gradient ascent is being performed for a function $\fg$
defined on the unit sphere, but the gradient described by $\grad \fg$
is for the function $\fg$ with domain $\R^{\Zpts}$. The more expanded
formula $\grad \fg(\vec u) - \vec u\vec u^T \grad \fg(\vec u)$ is the
projection of $\grad \fg$ onto the plane tangent to
$\sphere^{\Zpts-1}$ at $\vec u$.
This update keeps $\vec u$ near the sphere.

We may draw $\vec u$ uniformly at random from $\sphere^{\Zpts-1}\cap {\spn(C)^\perp}$ by first
drawing $\vec u$ from $\sphere^{\Zpts-1}$ uniformly at random, projecting $\vec u$ onto ${\spn(C)^\perp}$, and then
normalizing $\vec u$.  It is important that $\vec u$ stay near the orthogonal
complement of $\spn(C)$ in order to converge to a new cluster rather than
converging to a previously found optimum of $\fg$.
Step~\ref{alg1:step:optional} 
enforces this constraint during the update step.

\begin{myalgorithm}
\mycaption[\findenum]{
    Finds the local maxima of $\fg$ defined from the points
    $x_{i\vecdot}$ needed for clustering.
    The second input $\delta$ controls how far a point needs to be
    from previously found cluster centers to be a candidate future
    cluster center.} 
 \begin{algorithmic}[1]
  \Function{\findenum}{$X$, $\delta$}
  \State $C \leftarrow \{ \}$
  \While {$\abs{C} < \Zpts$}
  \State
    $j \leftarrow \argmax_i \{ \fg( \frac {x_{i \vecdot}}{\norm{x_{i \vecdot}}} )
      \suchthat \newline \phantom.\qquad\qquad\qquad
      \angle(\frac {x_{i \vecdot}}{\norm{x_{i \vecdot}}}, \vec u) > \delta \ 
      \forall \vec u \in C \}$
    \State $C \leftarrow C \cup \{\frac{x_{j \vecdot}}{\norm{x_{j \vecdot}}} \}$
  \EndWhile
  \State \Return $C$
  \EndFunction
 \end{algorithmic}\nopagebreak%
\end{myalgorithm}

In contrast to \findopt\@, \findenum\@ more directly uses
the point separation implied by the orthogonality of the approximate
cluster centers.  
Since each embedded data point should be near to a cluster center, the 
data points themselves are used as test points.
Instead of directly enforcing orthogonality between cluster means, a parameter $\delta > 0$ specifies the minimum allowable angle between found cluster means.

By pre-computing
the values of $\fg( x_{i \vecdot} / \norm{x_{i \vecdot}} )$ outside of the while
loop, \findenum\@ can be run in $O(\Zpts \dn^2)$ time.
\findenum\@ is likely to be slower than \findopt\@ which takes $O(\Zpts^2 \dn t)$ 
time where $t$ is the
average number of iterations to convergence.
The number of clusters $\Zpts$ cannot exceed (and is usually much smaller than) the number of graph vertices $\dn$.

\findenum\@ has a couple of nice features which may make it preferable on smaller data sets.
Each center found by \findenum\@ will always be within a cluster of data points
even when the optimization landscape is distorted under perturbation.
In addition, the maxima found by \findenum\@ are based on a more global outlook, which may be useful in the noisy setting.

\mysubsection{Choosing a Contrast Function}
There are many possible choices of contrast $g$ which are admissible for spectral clustering under Theorem~\ref{thm:complete_enumeration} including the following:
\begin{align*}
  g_{\tsig}(t) &= - \frac 1 {1 + \exp(-\abs{t})} & g_p(t) &= \abs{t}^p  \text{ where $p \in (2, \infty)$ }
\end{align*}
\begin{align*}
  g_{\tgau} &= e^{-t^2} &   g_{\tabs}(t) &= - \abs{t} & g_{\tht}(t) &= \log \cosh(t)
\end{align*}
In choosing contrasts, it is instructive to first consider the function $g_2(y) = y^2$ (which relaxes the criterion that $t \mapsto g(\sqrt{\abs t})$ be \emph{strictly} convex to plain convexity and is thus not admissible).  
Noting that $\fg[2](\vec u) = \sum_{i=1}^\edim w_i (\frac 1 {\sqrt{w_i}} \ipCanonical {\vec u} {\myZv_i} )^2 = 1$, we see that $\fg[2]$ is constant on the unit sphere.
We see that the distinguishing power of a contrast function for spectral 
clustering comes from our assumption that $t \mapsto g(\sqrt{\abs t})$ is strictly convex.
Intuitively, ``more strictly convex" contrasts have better resolving power but are also more sensitive to outliers and perturbations of the data.  Indeed, if $g$ grows too rapidly, a small number of outliers far from the origin could significantly distort the maxima structure of $\fg$.

Due to this tradeoff,  $g_{\tsig}$ and $g_{\tabs}$ could be important practical choices for the contrast function.
Both $g_{\tsig}(\sqrt{\abs t})$ and $g_{\tabs}(\sqrt{\abs t})$ have a strong convexity structure near the origin.
As $g_{\tsig}$ is a bounded function, it should be very robust to perturbations.
In comparison, $g_{\tabs}(\sqrt{\abs t}) = - \sqrt{\abs t}$ maintains a stronger convexity structure over a much larger region of its domain, and $g_{\tabs}(t)$ has only a linear rate of growth as $t \rightarrow \infty$.
This is a much slower growth rate than is present for instances in $g_p$ with $p > 2$.

\section{Clustering Experiments}\label{sec:experiments}

We now discuss our test results on our proposed spectral clustering
algorithms on a variety of real and simulated data.
The implementations for our spectral clustering algorithms are available on github: \url{https://github.com/vossj/HBR-Spectral-Clustering}.

\mysubsection{An Illustrating Example}

\figurename\@~\ref{fig:toy-ex-clusters} illustrates our function
optimization framework for spectral clustering.
In this example, random points $p_i$ were generated from 3 concentric
circles: 200 points were drawn uniformly at random from a radius 1
circle, 350 points from a radius 3 circle, and 700 points from a
radius 5 circle.
The points were then radially perturbed.
The generated points are displayed in \figurename\@~\ref{fig:toy-ex-clusters}~(a).
The similarity matrix $A$ was constructed as $a_{ij} =
\exp(-\frac 1 4 \norm{p_i - p_j}^2$), and the Laplacian embedding was performed using $L_{\rw}$.

\figurename\@~\ref{fig:toy-ex-clusters}~(b) depicts the clustering process with the contrast $g_{\tsig}$ on the resulting embedded points.
In this depiction, the embedded data sufficiently encodes the desired orthogonal basis structure that all local maxima of $F_{g_{\tsig}}$ correspond to desired clusters.
The value of $F_{g_{\tsig}}$ is displayed by the grayscale heat map on the unit sphere in \figurename\@~\ref{fig:toy-ex-clusters}~(b), with lighter shades of gray indicate greater values of $\fg[\tsig]$.
The cluster labels were produced using \findopt\@.
The rays protruding from the sphere correspond to the basis directions recovered by \findopt\@, and the recovered labels are indicated by the color and symbol used to display each data point.

\newlength{\figwidth}
\ifnum\version=1
\setlength{\figwidth}{0.7\linewidth}
\fi
\ifnum\version=2
\setlength{\figwidth}{\linewidth}
\fi
\begin{figure}[t]
  \centering
  \begin{minipage}{\figwidth}
    \begin{minipage}[t]{0.46\linewidth}
      \centering
      \includegraphics[width=\linewidth]{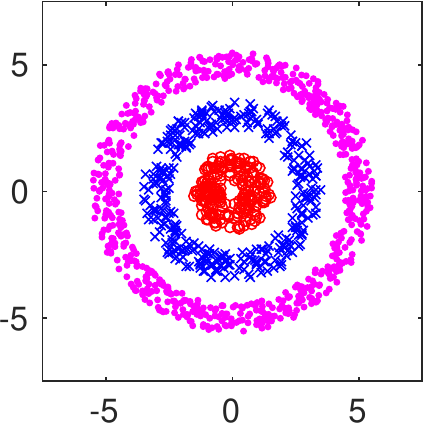}
      (a)
    \end{minipage}\hfill%
    \begin{minipage}[t]{0.52\linewidth}
      \centering
      \includegraphics[width=\linewidth]{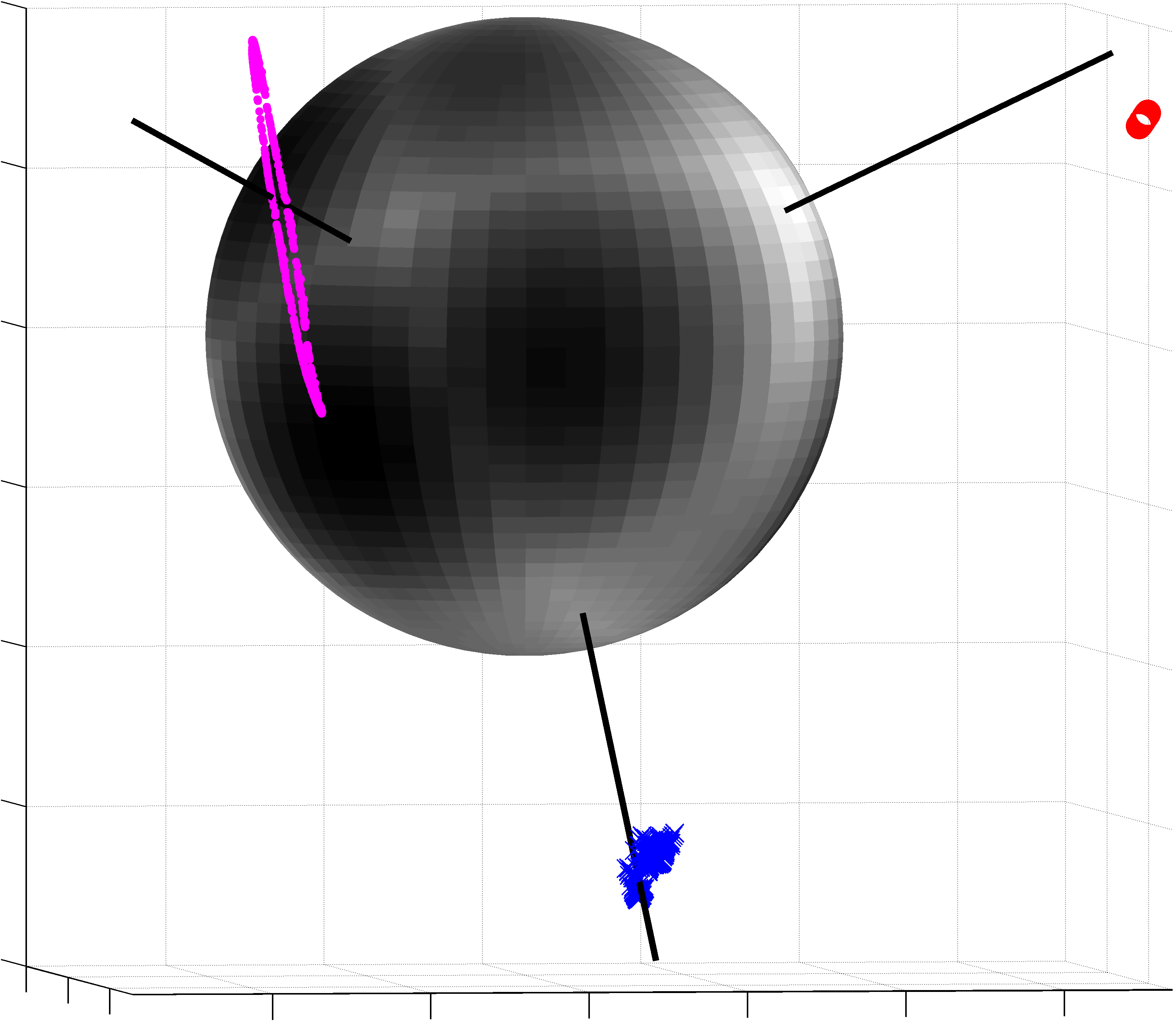}
      (b)
    \end{minipage}
  \end{minipage}
  \mycaption[Illustration of spectral clustering on concentric circle
  data]{\label{fig:toy-ex-clusters} An illustration of spectral
    clustering on the concentric circle data.
    (a) The output of clustering. (b) The embedded data, the
    directional evaluations of $F_{g_{\mathrm{sig}}}$, and the recovered basis
    for clustering.}
\end{figure}

\mysubsection{Image Segmentation Examples}\index{image segmentation}
\begin{figure}[tb]
  \centering%
  \begin{minipage}{\figwidth}
    \includegraphics[width=0.49\linewidth]{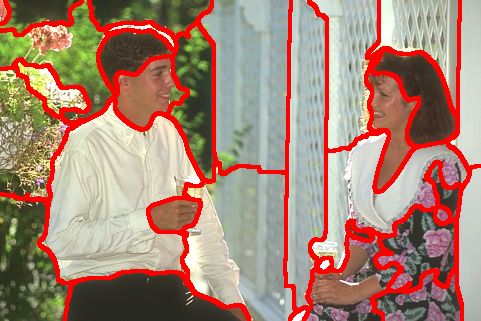}
    \includegraphics[width=0.49\linewidth]{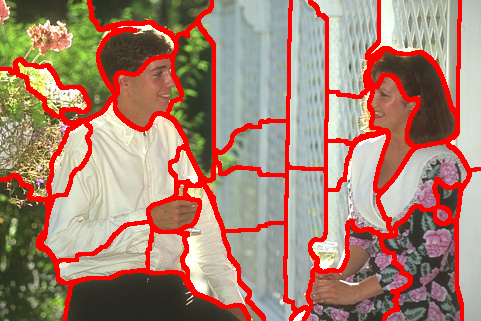}\vskip
    2pt%
    \noindent \includegraphics[width=0.49\linewidth]{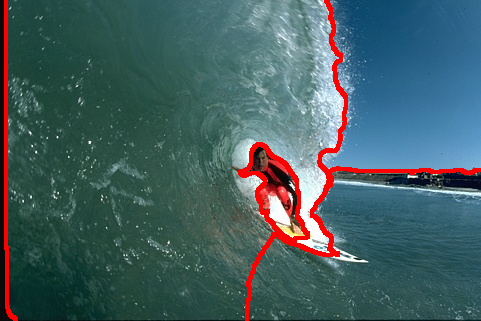}
    \includegraphics[width=0.49\linewidth]{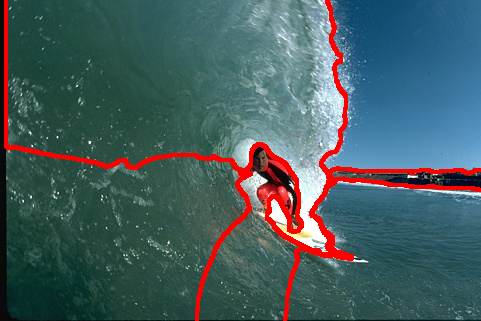}%
    \vskip 2pt%
    \noindent \includegraphics[width=0.49\linewidth]{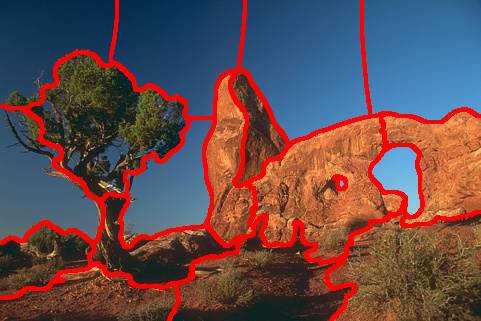}
    \includegraphics[width=0.49\linewidth]{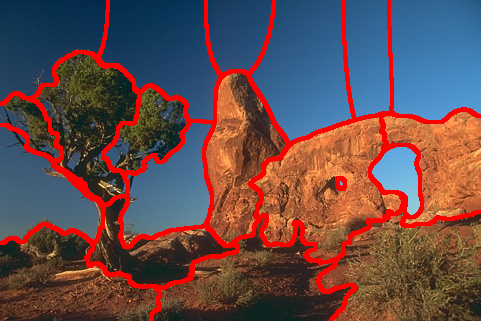}
  \end{minipage}    
  \mycaption[Image segmentations]{\label{fig:seg_images}
    Segmented images. 
    Segmentation using \findopt\@-$g_{\tabs}$ (left panels) compared to $k$-means (right panels).
    Red pixels mark the borders between segmented regions.}
\end{figure}
\begin{figure}[tb]
  \centering
  \begin{minipage}{1.0\figwidth}
    \centering
    \includegraphics[width=0.4\linewidth]{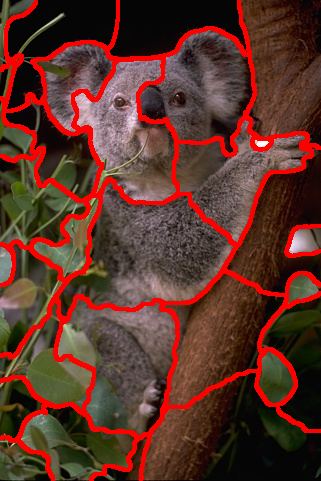}
    \includegraphics[width=0.4\linewidth]{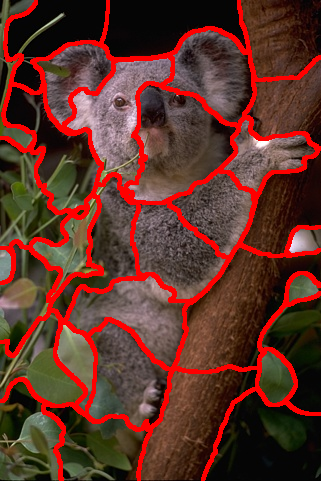}
    \vskip 2pt%
    \noindent
    \includegraphics[width=0.4\linewidth]{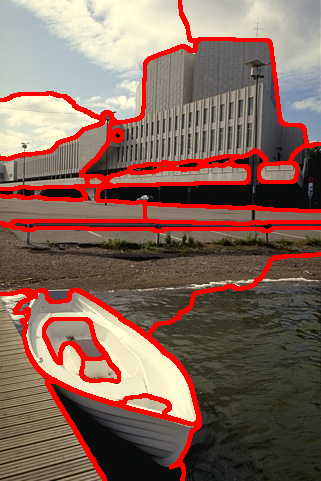}
    \includegraphics[width=0.4\linewidth]{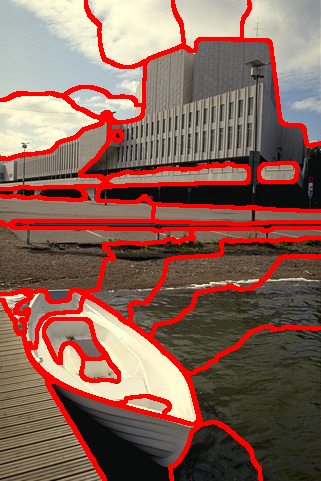}
  \end{minipage}
  \caption[More image segmentations]{Segmented images.
    Segmentation using \findopt\@-$g_{\tabs}$ (left panels) compared to
    $k$-means (right panels).
    Red pixels mark the borders between segmented regions.}
  \label{fig:seg_images2}
\end{figure}

Spectral clustering was first applied to image segmentation by \ifnatbib \citet{ShiMal00}, \else Shi and Malik \cite{ShiMal00}, \fi and it has remained a popular application of spectral clustering.
The goal in image segmentation is to divide an image into regions which represent distinct objects or features of the 
image.
\figurename~\ref{fig:seg_images} and \figurename~\ref{fig:seg_images2} show several segmentations produced by \findopt\@-$g_{\tabs}$ and spherical $k$-means on several example images from the BSDS300 test set~\citep{MartinFTM01}.

For this example application, we used a relatively simple notion of similarity
based only on the color and proximity of the image's pixels.
Let $p_i$ denote the $i$\textsuperscript{th} pixel.
Each $p_i$ has a location $\vec x_i$ and an RGB color $\vec c_i=(r_i,
g_i, b_i)^T$.
We used the following similarity between any two distinct pixels $p_i$
and $p_j$:
\begin{equation} \label{eq:img-similarity}
  a_{ij} =
  \begin{cases}
    e^{-\frac 1 {\alpha^2} \norm{\vec x_i-\vec x_j}^2 }e^{-\frac 1 {\beta^2} \norm{\vec c_i - \vec c_j}^2}  & \text{if $\norm{\vec x_i - \vec x_j} < R$} \\* 
    0 & \text{if $\norm{\vec x_i - \vec x_j} \geq R $}    
  \end{cases}
\end{equation}
for some parameters $\alpha$, $\beta$, and radius $R$.  By enforcing that $a_{ij}$ is 0 for
points which are not too close, we build a sparse similarity
matrix\index{similarity matrix!sparse} which greatly speeds up the computations.  As the similarity measure decays exponentially with distance, the zeroed entries would be very small anyway.

Determining the number of clusters to use in spectral clustering is an
unsolved problem.  However, the BSDS300 data set includes hand labeled
segmentations.  From the hand labeled segmentations for a particular image, one
human segmentation was chosen at random and the number of segments from that segmentation was used
as the number of clusters $\Zpts$ for spectral clustering to search for.  
No other information from the human segmentations was used in generating the image segmentations.

In order to reduce the effect of salt and pepper type noise,
the images were preprocessed using $9 \times 9$ median filtering prior to constructing the similarity matrices.
The similarity from equation~\eqref{eq:img-similarity} was constructed with common fixed values of $\alpha$, $\beta$, and $R$ across all images. 
Spectral clustering was performed using \findopt\@ under the contrast function $g_{\tabs}$ and the $L_{\rw}$ embedding.

Qualitatively, we found that for the same embedding, $k$-means is more likely to over segment large regions within the image, in effect balancing the cluster sizes.
In contrast, our proposed \findopt\@ algorithm tended to segment out additional small regions within the image more frequently.

\mysubsection{Stochastic Block Model with Imbalanced Clusters}
We construct a similarity graph $A = \diag(A_1, A_2, A_3) + E$ where each $A_i$ is a symmetric matrix corresponding to a cluster and $E$ is a small perturbation.
We set $A_1 = A_2$ to be $10 \times 10$ matrices with entries $0.1$.
We set $A_3$ to be a $1000 \times 1000$ matrix which is symmetric, approximately 95\% sparse with randomly chosen non-zero locations set to 0.001.
When performing this experiment $50$ times, \findopt\@-$g_{sig}$ obtained a mean accuracy of 99.9\%.
In contrast, spherical $k$-means with randomly chosen starting points obtained a mean accuracy of only 42.1\%. 
It turns out that splitting the large cluster is in fact optimal in terms of the spherical $k$-means objective function but leads to poor classification performance.
Our method does not suffer from that shortcoming.

\mysubsection{Performance Evaluation on UCI Datasets}
\begin{table*}
  \centering
  \ifnum\version=2
  \mycaption[Spectral clustering accuracy
  comparison]{\label{tab:accuracy-comparison} Percentage accuracy of
    spectral clustering algorithms, with the best performing non-oracle
    algorithm bolded.
  }
  \fi
  \ifnum\version=1 \small \fi
  \begin{tabular}{|r|c|c|c|c|c|c|c|c|c|c|c|c|} \hline
    & oracle-        & $k$-means           & \multicolumn{5}{c|}{\findopt\@} & \multicolumn{5}{c|}{\findenum\@} \\ \cline{4-13}
    & centroids & cosine & $g_{\tabs}$ & $g_{\tgau}$ & $g_3{\phantom '}$ & $g_{\tht}$ & $g_{\tsig}$ & $g_{\tabs}$ & $g_{\tgau}$ & $g_{3}^{\phantom '}$ & $g_{\tht}$ & $g_{\tsig}$  \\ \hline
    E.\@ coli & 79.7 & 69.0 & 80.9 & 81.2 & 79.3 & 81.2 & 80.6 & 68.7 & \textbf{81.5} & \textbf{81.5} & 68.7 & \textbf{81.5}  \\
    flags & 33.2 & 33.1 & \textbf{36.8} & 34.1 & 36.6 & \textbf{36.8} & 34.4 & 34.7 & \textbf{36.8} & \textbf{36.8} & 34.7 & \textbf{36.8}  \\
    glass & 49.3 & 46.8 & \textbf{47.0} & 46.8 & \textbf{47.0} & \textbf{47.0} & 46.8 & \textbf{47.0} & \textbf{47.0} & \textbf{47.0} & \textbf{47.0} & \textbf{47.0}  \\
    Iris & 84.0 & \textbf{84.0} & 82.8 & 83.4 & 78.5 & 83.4 & 83.2 & 67.3 & 83.3 & 83.3 & 71.3 & \textbf{84.0} \\
    thyroid & 72.4 & 80.4 & \textbf{82.4} & 81.3 & 82.2 & 82.2 & 81.5 & 81.8 & {82.2} & {82.2} & 81.8 & {82.2} \\
    car eval & 56.1 & 36.4 & {37.0} & 36.3 & 36.3 & 35.2 & 36.6 & 49.6 & 32.3 & 41.1 & \textbf{49.9} & 41.1  \\
    cell cycle & 74.2 & 62.7 & 64.3 & 64.4 & 63.8 & {64.5} & 64.0 & 60.1 & 62.9 & \textbf{64.8} & 61.1 & 62.7  \\ \hline     
  \end{tabular}
  \ifnum\version=1
  \mycaption[Spectral clustering accuracy
  comparison]{\label{tab:accuracy-comparison} Percentage accuracy of
    spectral clustering algorithms, with the best performing non-oracle
    algorithm bolded.
  }
  \fi
\end{table*}



We compare spectral clustering performance on a number of data sets with unbalanced cluster sizes.
In particular, we use the E.\@ coli, flags, glass, Iris, thyroid disease, and car evaluation data sets which are part of the UCI machine learning repository~\citep{Bache+Lichman:2013}.
We also use the standardized gene expression data set~\citep{yeung2001model,Yeung01supplement}, which is also referred to as cell cycle.
For the flags data set, we used religion as the ground truth labels, and for thyroid disease, we used the new-thyroid data.

For all data sets, we only used fields for which there were no missing
values, we normalized the data such that every field had unit standard
deviation, and we constructed the similarity matrix $A$ using a
Gaussian kernel\index{Gaussian kernel} $k(\vec y_i, \vec y_j) = \exp( - \alpha \norm{\vec y_i
  - \vec y_j}^2 )$.
The parameter $\alpha$ was chosen separately for each data set in order to create a good embedding.
The choices of $\alpha$ were:  0.25 for E.\@ coli, 32 for glass, 0.5 for Iris, 32 for thyroid disease, 128 for flags, 0.25 for car evaluation, and 0.125 for cell cycle.

The spectral embedding was performed using the symmetric normalized Laplacian $L_{\sym}$.
Then, the clustering performance of our proposed algorithms \findopt\@ and \findenum\@ (implemented with $\delta= 3\pi / 8$ radians)
were compared with the following baselines: 
\begin{itemize}
\item 
oracle-centroids: The means $\mu_j = \frac 1 {\abs {\calS_j}} \sum_{i \in \calS_j} \frac{x_{i \vecdot}}{\norm{x_{i\vecdot}}}$ are set using the ground truth labels for all $j \in [\edim]$.
Points  are assigned to their nearest cluster mean in cosine distance.
\item 
  $k$-means-cosine: Spherical $k$-means (standard matlab kmeans library function called using the cosine distance and using the default $k$-means++ mean initialization) is run with a random initialization of the means, (cf.\@ \cite{ng2002spectral}).
\end{itemize}

We report the clustering accuracy of each algorithm in Table~\ref{tab:accuracy-comparison}.
The accuracy is computed using the best matching between the clusters and the true labels.
The reported results consist of the mean performance over a set of 25 runs for each algorithm.
The number of clusters being searched for was set to the ground truth number of clusters.
In most cases, our proposed algorithms show improvement in performance over spherical $k$-means.


\section{Basis Recovery With Each Laplacian Embedding}
\label{app:Lsym}

We have already argued that graph Laplacians $L$ and $L_{\rw}$ can be used for spectral clustering within our BEF framework, and we have asserted that $L_{\sym}$ can also be used.
We now discuss how orthogonal BEF recovery can be used for spectral clustering in the setting where $G$ consists of $\Zpts$ connected components using any of the graph Laplacians.
First, in section~\ref{sec:Lsym-null-space}, we show how the Laplacian embedding for the symmetric normalized Laplacian $L_{\sym}$ differs and generalizes upon the embedding structure arising for $L$ and $L_{\rw}$ (cf.\@ Proposition~\ref{prop:discrete-simplex}).
Then, in section~\ref{sec:spectral-contrast-admissibility}, we prove that the spectral embedding induced by any of the discussed graph Laplacians gives rise to an optimization problem on $\sphere^{\edim - 1}$ in which the local maxima enumerate the desired clusters for spectral clustering.
More precisely, we prove a generalization of Theorem~\ref{thm:complete_enumeration} which includes embeddings generated using $L$, $L_{\rw}$, and $L_{\sym}$.

The discussion in this section highlights the differences between
using $L_{\sym}$ and using either $L$ or $L_{\rw}$ for the proposed
spectral algorithms.
Whereas taking an orthogonal basis of
$\NN(L)$ or $\NN(L_{\rw})$ produces embedded points which are orthogonal and of fixed norm within any particular class, using 
$\NN(L_{\sym})$ produces embedded points along perpendicular rays but with varying intra-class norms as will be seen in
Proposition~\ref{app_sym:prop:ray_structure}.  
Despite these differences, when given a contrast function $g$ meeting the strict convexity criterion from Assumption~\ref{assump:strict-convexity}, the proposed algorithms \findopt\@ and \findenum\@ which worked for spectral clustering using $L$ and $L_{\rw}$ also work for spectral clustering using $L_{\sym}$.

\subsection{Null Space Structure of the Normalized Laplacians}\label{sec:Lsym-null-space}
We now investigate the null space\index{graph Laplacian!symmetric normalized!null space} structure of the normalized graph Laplacians.
We will first describe the null space structures $L_{\sym}$ and $L_{\rw}$ for a graph $G$ consisting of $\edim$ components.
Then, we will show how the null space structures of $L_{\sym}$, $L_{\rw}$, and $L$ can all be viewed within a single, more generalized notion of a graph embedding.

Let $G = (V, A)$ be an $\dn$-vertex graph containing $\Zpts$ connected components such that the 
$i$\textsuperscript{th} component has vertices with indices in the 
set $\calS_i$.  For any set $C \subset V$, we define
\begin{equation}\label{eq:def_delta_set}
  \delta_\myD (C) := \sum_{i \in C} \myd_{ii} \ .
\end{equation}
where $\myD = \diag(\myd_{11}, \myd_{22}, \dotsc, \myd_{\dn\dn})$ is a diagonal matrix with strictly positive entries.
For now, we will take $\myD$ to be the diagonal degree matrix $D$ such that $\myd_{ii} = d_{ii} = \sum_{j=1}^\dn a_{ij}$.
Then, $\delta_D(C)$ is the sum of vertex degrees for vertices in the set $C$.
Using this definition, we are able to characterize the embedding structure of $L_{\sym}$.

\begin{prop} \label{app_sym:prop:ray_structure}
  Let $G$ be a similarity graph consisting of $\Zpts$ connected components for which $L_{\sym}$ is well defined.  Let the 
  vertex indices be partitioned into sets $\calS_1, \dotsc, \calS_{\Zpts}$ corresponding to the $\Zpts$ connected components. 
  Then, 
  $\dimop(\NN(L_{\sym})) = \Zpts$.  If $X = (x_{\vecdot 1},
  \dotsc, x_{\vecdot \Zpts})$ contains a scaled basis of $\NN(L_{\sym})$
  in its columns such that $\norm{x_{\vecdot i}} = \sqrt \dn$,
  then there exist
  $\Zpts$ mutually orthogonal unit vectors $\myZv_1, \dotsc, \myZv_{\Zpts}$ such that
  whenever $i \in \calS_j$, the row vector 
  \begin{equation}\label{eq:sym-embpoint-formula}
    x_{i \vecdot} =  \sqrt{\dn d_{ii}\delta_D(\calS_j)^{-1}}\myZv_j^T \ .  
  \end{equation}
\end{prop}
\begin{proof}
An important property of the symmetric normalized Laplacian 
\citep[Proposition 3]{von2007tutorial} is that
for all $\vec u \in \R^\dn$,
\begin{equation}\label{app:eq:Laplacian}
  \vec u^T L_{\sym} \vec u = \frac 1 2 \sum_{i, j  \in V} a_{ij}\left(
      \frac{u_i}{d_{ii}^{\sfrac 1 2}} - \frac{u_j}{d_{jj}^{\sfrac 1 2}} \right)^2 \ .
\end{equation}
$L_{\sym}$ is positive semi-definite, and $\vec u$ is a 
0-eigenvector of $L_{\sym}$ if and only if plugging $\vec u$ into equation~\eqref{app:eq:Laplacian} yields 0.
Let $\vec y_{\calS_j}^{\phantom{*}}$ be the vector such that
\begin{equation}\label{eq:Lsym-eigenvecs-ideal}
  \vec y_{\calS_j}^{\phantom{*}} = \left\{ \begin{array}{ll}
                d_{ii}^{\sfrac 1 2} & \text{if $i \in \calS_j$.} \\
                0       &       \text{otherwise}
              \end{array} \right. \ .
\end{equation}
Then, $B = (\delta_D(\calS_1)^{-\sfrac 1 2} \vec y_{\calS_1}^{\phantom{*}}, \dotsc, \delta_D(\calS_d)^{-\sfrac 1 2} 
\vec y_{\calS_d}^{\phantom{*}} )$ contains
an orthonormal basis for $\NN(L_{\sym})$ in its columns.

Defining $M_{\calS_i} = \vec y_{\calS_i}^{\phantom{*}}\vec y_{\calS_i}^T$, we get:
\begin{equation} \label{app_sym:eq:P_ker_L}
  P_{\NN(L)} = BB^T = \sum_{i = 1}^\Zpts \delta_D(\calS_i)^{-1} M_{\calS_i} \ .
\end{equation}
But $P_{\NN(L_{\sym})}$ can be constructed from any orthonormal 
basis of $\NN(L_{\sym})$. 
In particular, $P_{\NN(L_{\sym})} = \frac 1 \dn XX^T$ as well.  Hence,
$\frac 1 \dn \ipCanonical{x_{i \vecdot}}{x_{j \vecdot}} = (P_{\NN(L)})_{ij} = 
\delta_D(\calS_\ell)^{-1}d_{ii}^{\sfrac 1 2}d_{jj}^{\sfrac 1 2}$ precisely
when there exists $\ell \in [\Zpts]$ such that $i, j \in \calS_\ell$.  Otherwise, 
$x_{i \vecdot} \perp x_{j \vecdot}$.

Note that for $i, j \in \calS_\ell$,
\begin{align*}
 &\cos(\angle(x_{i \vecdot}, x_{j, \vecdot})) 
 = \frac{\ipCanonical{x_{i \vecdot}}{x_{j \vecdot}}}{\norm{x_{i \vecdot}} \norm{x_{j \vecdot}}} \\
 &\quad= \frac{ \dn \delta_D(\calS_\ell)^{-1}_{\phantom{*}} d_{ii}^{\sfrac 1 2} d_{jj}^{\sfrac 1 2} }{(\dn^{\sfrac 1 2} \delta_D(\calS_\ell)^{\sfrac {-1} 2}_{\phantom{*}} d_{ii}^{\sfrac 1 2})(\dn^{\sfrac 1 2}_{\phantom{*}} \delta_D(\calS_\ell)^{\sfrac {-1} 2}_{\phantom{*}} d_{jj}^{\sfrac 1 2})} = 1 \ .
\end{align*} 
Thus, points from the same cluster lie on the same ray from the origin.
It follows that there are $\Zpts$ mutually orthogonal unit vectors, $\myZv_1, \dotsc, \myZv_{\Zpts}$ such that  $x_{i \vecdot} =\sqrt{\dn d_{ii} \delta_D(\calS_{\ell})^{-1}} \myZv_{\ell}^T$ for each $i \in \calS_{\ell}$.
\end{proof}

We will make use of the close connection between the eigenvector structure of $L_{\rw}$ and $L_{\sym}$ in order to characterize the Laplacian embedding structure of $L_{\rw}$.
The following fact can be found in the tutorial \citep[Proposition 3]{von2007tutorial}.
\begin{fact}\label{fact:L-sym-rw-eig-relation}
  $(\lambda, \vec u)$ is an eigenvalue-eigenvector pair of $L_{\rw}$
  if and only if $(\lambda, D^{1/2} \vec u)$ is an
  eigenvalue-eigenvector pair for $L_{\sym}$. 
\end{fact}
By using Fact~\ref{fact:L-sym-rw-eig-relation} and Proposition~\ref{app_sym:prop:ray_structure}, we obtain the embedding structure for $L_{rw}$.\index{graph Laplacian!asymmetric normalized!null space}
\begin{prop}\label{prop:rw-embedding}
  Let the similarity graph $G=(V, A)$ contain $\Zpts$ connected components with indices in the sets $\calS_1, \dotsc, \calS_{\Zpts}$, 
  let $\dn = \abs{V}$, and let 
  $L_{\rw}$ be well defined for $G$.
  Then, $\NN(L_{\rw})$ has dimensionality $\Zpts$.
  Let $X = (x_{\vecdot 1},  \dotsc, x_{\vecdot \Zpts})$ contain $\Zpts$ scaled, 
  orthogonal column vectors forming a basis of $\NN(L_{\rw})$ such that $\norm{x_{\vecdot j}} = \sqrt \dn$ for each $j \in [\Zpts]$.
  Then, there exist weights $w_1, \dotsc, w_{\Zpts}$ with 
  $w_j = \frac{\abs{\calS_j}}{\dn}$ 
  and
  mutually orthogonal vectors $\myZv_1, \dotsc, \myZv_{\Zpts} \in \R^{\Zpts}$ such that
  whenever $i \in \calS_j$, the row vector $x_{i \vecdot} = \frac 1 {\sqrt{w_j}} \myZv_j^T$.
\end{prop}
\begin{proof}
  By Fact~\ref{fact:L-sym-rw-eig-relation}, we may construct an
  orthogonal basis of $\NN(L_{\rw})$ using a particular choice of
  orthogonal basis of $\NN(L_{\sym})$.
  In particular, we define the vectors $\vec y_{\SS_j}$ the same as in the proof of Proposition~\ref{app_sym:prop:ray_structure}, and we obtain that the vectors $\tilde {\vec y}_{\SS_j}^{\phantom *} := D^{-1/2} \vec y_{\SS_j}^{\phantom *}$ are $0$-eigenvectors of $L_{\rw}$.
  Using equation~\eqref{eq:Lsym-eigenvecs-ideal}, we see that $\vec y_{\SS_j}^{\phantom *} = \One_{\SS_j}$.
  In particular, it follows that $\{ \abs{\SS_1}^{-\sfrac 1 2} \One_{\SS_1}, \dotsc, \abs {\SS_\edim}^{-\sfrac 1 2} \One_{\SS_\edim} \}$ is an orthonormal basis of $\NN(L_{\rw})$.
  From the discussion around
  equation~\eqref{ch-opt:eq:Laplacian-null-space}, it follows that
  $\NN(L)$ and $\NN(L_{\rw})$ are the same space in this setting where
  $G$ consists of $\edim$ connected components.
  Our desired result thus follows from Proposition~\ref{prop:discrete-simplex}.
\end{proof}

We note that the Propositions~\ref{prop:discrete-simplex}, \ref{app_sym:prop:ray_structure}, and \ref{prop:rw-embedding} are closely.
From Propositions~\ref{prop:discrete-simplex} and \ref{prop:rw-embedding}, we see that $L$ and $L_{\rw}$ give rise to the same embedding structure when $G$ consists of $\edim$ connected components.
Further, we may place the embedding structure for $L$ (or equivalently $L_{\rw}$) into the notation used for describing the ray structure of $L_{\sym}$.
In particular, if we let $\myD = \Id$, we see that $\delta_\Id(\SS_j) = \abs{\SS_j}$.
Recalling that $w_j = \frac{\abs{\SS_j}}{\dn}$, we see (by replacing
$D$ with $\Id$ in equation~\eqref{eq:sym-embpoint-formula}) that
$\sqrt{\dn \Id_{ii} \delta_\Id (\SS_j)^{-1}} \myZv_j^T = \frac 1
{\sqrt {w_j}}\myZv_j^T$, which is the required replacement to recreate
the statements of Proposition~\ref{prop:discrete-simplex} and
Proposition~\ref{prop:rw-embedding}.
In particular, we may create a generalized notion of a graph embedding which captures all of the Laplacian embeddings.
\begin{defn}
  Let $G$ be a similarity graph consisting of $\dn$ vertices and $\edim$ connected components such indices partitioned into sets $\SS_1, \dotsc, \SS_\edim$ corresponding to the connected components.
  Let $\myD = \diag(\myd_{11}, \dotsc, \myd_{\dn \dn})$ be a positive definite matrix.
  Let $\varphi$ be a map which takes the $i$\textsuperscript{th} vertex of $G$ to a point $\vec x_i \in \R^\edim$.
  If there exists an orthonormal basis $\myZv_1, \dotsc, \myZv_\edim$ of $\R^\edim$ such that $\vec x_i = \sqrt{\dn \myd_{ii} \delta_{\myD}(\SS_j)^{-1}} \myZv_j$ for each $i \in \SS_j$, then we call $\varphi$ a \emph{($G$, $\myD$)-orthogonal embedding}.\index{spectral embedding!($G$, $\myD$)-orthogonal embedding}
\end{defn}
We see by Proposition~\ref{app_sym:prop:ray_structure}, the Laplacian embedding induced by $L_{\sym}$ is a $(G, D)$-orthogonal embedding; and by Propositions~\ref{prop:discrete-simplex} and~\ref{prop:rw-embedding}, the Laplacian embedding induced by $L$ and $L_{\rw}$ are $(G, \Id)$-orthogonal embedding.

\subsection{Maxima Structure of the Resulting BEFs}\label{sec:spectral-contrast-admissibility}
In this section, we demonstrate that by performing function maximization over the directional projections of embedded data arising from any of the Laplacian embeddings, we are able to recover the desired clusters for spectral clustering.
We will make use of the following construction.

\begin{construction}\label{constr:spectral-Fg-general}
  Let $G$ be a similarity graph consisting of $\dn$ vertices and
  $\edim$ connected components with indices partitioned into the sets
  $\SS_1, \dotsc, \SS_\edim$.
  We suppose that $\myD = \diag(\myd_{11}, \dotsc, \myd_{\dn \dn})$ is
  a positive definite matrix.
  We suppose that $\vec x_1, \dotsc, \vec x_\dn$ is a $(G,
  \myD)$-orthogonal embedding of the vertices of $G$ such that $\vec
  x_i = \sqrt{\dn \myd_{ii} \delta_{\myD}(\SS_j)^{-1}} \myZv_j$ for each
  $i$ in $\SS_j$.
  Parallel to the text of section~\ref{sec:arb_functions}, we
  construct a function $\fg: \sphere^{\Zpts-1} \rightarrow \R$ from a
  continuous contrast function $g:[0, \infty) \rightarrow \R$ where it
  is assumed that $t \mapsto g(\sqrt t)$ is strictly convex (cf.\@
  Assumption~\ref{assump:strict-convexity}).
  We construct $\fg$ as
  \begin{align} \label{eq:app:def_f}
    \fg(\vec u) 
    &:= \frac 1 \dn \sum_{i=1}^\dn g(\abs{\ipCanonical {\vec u}{\vec x_i}}) \notag \\
    &= \frac 1 \dn \sum_{j=1}^\Zpts \sum_{i \in \calS_j}g\left(\norm{\vec x_{i}} \cdot \Abs{\ipCanonical{\vec u}{\myZv_j} }\right) \ .
  \end{align}
\end{construction}
First, we make a couple of comments about
Construction~\ref{constr:spectral-Fg-general}.
Using the discussion at the end of section~\ref{sec:Lsym-null-space},
when $\myD = \Id$ the embedded points $\vec x_i$ can be obtained from
the rows of $X$ in Proposition~\ref{prop:discrete-simplex}, and they
thus correspond to the embedded points arising from $L$.
For this choice of $\myD = \Id$,
Construction~\ref{constr:spectral-Fg-general} is thus a strict
generalization of Construction~\ref{constr:spectral-Fg-L}.
However, Construction~\ref{constr:spectral-Fg-general} also captures $L_{\rw}$ (with $\myD = \Id$ by Proposition~\ref{prop:rw-embedding}) and $L_{\sym}$ (with $\myD = D$ by Proposition~\ref{app_sym:prop:ray_structure}).

We now wish to generalize Theorem~\ref{thm:complete_enumeration} by
showing that the local maxima of $\fg$ from
Construction~\ref{constr:spectral-Fg-general} are precisely the
directions $\pm \myZv_1, \dotsc, \pm \myZv_\edim$.
We will first argue that $\fg$ has no extraneous maxima, and then that
the direction $\pm \myZv_1, \dotsc, \pm \myZv_\edim$ actually are
maxima.
To see that $\fg$ has no extraneous local maxima, we need only
demonstrate that $\fg$ is an orthogonal BEF satisfying
Assumption~\ref{assump:strict-convexity} and apply
Theorem~\ref{ch-opt:thm:hbfopt_strict_conv}.

\begin{lem} \label{lem:app:lsym_no_extraneous_maxima}
  Let $\fg$ and $\myZv_1, \dotsc, \myZv_\edim$ be as in Construction~\ref{constr:spectral-Fg-general}.
  Then, the local maxima of $\fg$ is contained in the set $\{\pm \myZv_i \suchthat i \in [\Zpts]\}$.
\end{lem}
\begin{proof}
  Define $g_i: \R \rightarrow \R$ by $g_i(t) := \frac 1 \dn \sum_{j \in \calS_i}g(\norm{\vec x_j} \cdot \abs t)$.  Using equation~\eqref{eq:app:def_f}, we obtain
  \begin{equation*}
    \fg(\vec u) 
    = \frac 1 \dn \sum_{i=1}^\Zpts \sum_{j \in \calS_i}g\left(\norm{\vec x_j} \cdot \Abs{\ipCanonical{\vec u}{\myZv_i} }\right)
    = \sum_{i=1}^\Zpts g_i(\ipCanonical{\vec u}{ \myZv_i}) \ .
  \end{equation*}
  Since $t \mapsto g(\sqrt t)$ is strictly convex, it follows that $t \mapsto g_i(\sign(t)\sqrt{\abs t})$ is
  strictly convex for all $i \in [\edim]$.  By Theorem~\ref{ch-opt:thm:hbfopt_strict_conv}, the local maxima of $\fg$ are contained in $\{\pm  \myZv_i : i \in [\Zpts]\}$.
\end{proof}

What remains to be seen is that the directions $\{ \pm  \myZv_i \suchthat i \in [\Zpts] \}$
are local maxima of $\fg$.  
For notational simplicity, we identify $\myZv_1, \dotsc, \myZv_{\edim}$ with the
canonical directions $\vec e_1, \dotsc, \vec e_\edim$ in an unknown coordinate system so that $u_i$ is shorthand for $\ipCanonical{\vec u}{\vec e_i}$.
In our proofs, we exploit the convexity structure induced by the change of variable introduced in the proof of Theorem~\ref{ch-opt:thm:hbfopt_strict_conv}, namely $\psi$ 
defined by $\psi_i(\vec u) := u_i^2$ which maps the domain $\sphere^{\edim - 1}$ onto the simplex $\Delta^{\Zpts-1} := \conv(\vec e_1, \dotsc, \vec e_\edim)$.

\begin{lem} \label{lem:app:convex_inner_sufficiency} Let $\vec x_1,
  \dotsc, \vec x_\dn$ be as in
  Construction~\ref{constr:spectral-Fg-general} with the added
  assumption that $\myZv_i = \vec e_i$ for each $i \in [\edim]$.
  Let $h:[0, \infty)\rightarrow \R$ be a strictly convex function.  
  Let $H:\Delta^{\edim-1} \rightarrow \R$ be given by $H({\vec u}) =
  \frac 1 \dn \sum_{i=1}^{\edim} \sum_{j \in \calS_i} h(u_i \norm{\vec
    x_j}^2)$.
  Then the set $\{\vec e_i \suchthat i \in [\Zpts]\}$ is contained in
  the set of strict local maxima of $H$.
\end{lem}

\begin{proof}
  By the symmetries of $H$, it suffices to show that $\vec e_1$ is a strict local
  maximum of $H$.  To see this, choose ${\vec u} \neq \vec e_1$ from a neighborhood
  of $\vec e_1$ relative to $\Delta^{\edim - 1}$ to be specified later.  
  Let $\Lambda_{\vec u} = \{ i \suchthat i \in [\edim] \setminus \{1\}, u_i \neq 0\}$.  Then,
  \begin{align*}
    &H(\vec e_1) - H({\vec u}) \\&= 
    \frac 1 \dn \left[ \sum_{j \in \calS_1} h(\norm{\vec x_j}^2) + \sum_{i=2}^{\Zpts}\sum_{j \in \calS_i} h(0) - \sum_{i = 1}^\Zpts \sum_{j \in \calS_i} h(u_i\norm{\vec x_j}^2) \right] \\
    & = 
    \frac 1 \dn \left[ \sum_{j \in \calS_1} \left(h(\norm{\vec x_j}^2) - h(u_1\norm{\vec x_j^2}) \right) \right. \\*
      &\qquad\qquad\left. - \sum_{i = 2}^\Zpts \sum_{j \in \calS_i}\left( h(u_i\norm{\vec x_j}^2) - h(0) \right) \right] \\
    & = 
    \frac 1 \dn \left[ \sum_{j \in \calS_1} \norm{\vec x_j}^2(1-u_1)\frac{h(\norm{\vec x_j}^2) - h(u_1\norm{\vec x_j^2})}{\norm{\vec x_j}^2(1-u_1)} \right. \\*
    &\qquad\quad \left.
     - \sum_{i \in \Lambda_u} \sum_{j \in \calS_i}u_i\norm{\vec x_j}^2\frac{h(u_i\norm{\vec x_j}^2) - h(0)}{u_i\norm{\vec x_j}^2}  \right] \ .
  \end{align*}
  We have written $H(\vec e_1) - H({\vec u})$ as a weighted sum of difference quotients (slopes).  
  We would like to apply Lemma~\ref{lem:convex-diff-quotient-increasing} in order to demonstrate that there is a neighborhood $B$ of $\vec e_1$ relative to $\Delta^{\edim - 1}$ such that ${\vec u} \in B \setminus \{\vec e_1\}$ implies $H(\vec e_1) - H({\vec u}) > 0$.
  First, we notice that for each $\vec x_j$, ${\vec u}$ breaks the interval into left and right pieces, yielding two slopes of interest:
  \begin{equation*}
    \lslope[ij](\vec u) = \frac{h(u_i\norm{\vec x_j}^2)-h(0)}{u_i\norm{\vec x_j}^2}
  \end{equation*}
  and
  \begin{equation*}
    \rslope[ij](\vec u) = \frac{h(\norm{\vec x_j}^2) - h(u_i\norm{\vec x_j}^2)}{\norm{\vec x_j}^2(1-u_i)} \ .
  \end{equation*}
  Thus, 
  \begin{align*}
    H(\vec e_1) - H({\vec u})
    &= \frac 1 \dn \left[ \sum_{j \in S_1} \norm{\vec
        x_j}^2 (1 - u_1) \rslope[1j](\vec u) \right. \\
    &\qquad\qquad\left.
      - \sum_{i \in \Lambda_{\vec u}}
      \sum_{j \in S_i} u_i \norm{\vec x_j}^2 \lslope[ij](\vec u) \right] \ .
  \end{align*}
  Let $B = \{{\vec u} \suchthat u_i < \frac{\min_j \norm{\vec x_j}^2}{\max_j \norm{\vec x_j}^2} \text{ for all } i \neq 1\} \setminus \{ \vec e_1 \}$.  
  Then, fixing ${\vec u} \in B$ and $i \neq 1$, we have that $u_i\norm{\vec x_{j_1}}^2 < \norm{\vec x_{j_2}}^2$ for any $j_1 \in \calS_i$ and $j_2 \in \calS_1$.
  Let $\lslope[\max](\vec u) := \max \{ \lslope[ij](\vec u) \suchthat i \in \Lambda_{\vec u}, j \in \calS_i \}$
  and $\rslope[\min](\vec u) := \min \{ \rslope[1j](\vec u) \suchthat j \in \calS_1 \}$.
  From Lemma~\ref{lem:convex-diff-quotient-increasing}, it follows that $\lslope[\max](\vec u) < \rslope[\min](\vec u)$ for all $\vec u \in B$.
  Thus,
  \begin{align*}
    H(\vec e_1) - H({\vec u}) &\geq \frac 1 \dn \left[ \sum_{j \in S_1} \norm{\vec x_{j}}^2 (1 - u_1) \rslope[\min](\vec u) \right. \\*
    &\left. \qquad\qquad
      - \sum_{i \in \Lambda_{\vec u}} \sum_{j \in S_i} {\vec u}_i \norm{\vec x_{j}}^2 \lslope[\max](\vec u) \right] \\
    &= (1 - u_1) \rslope[\min](\vec u) -\sum_{i = 2}^{\edim} u_i \lslope[\max](\vec u) \\
    &= (1-u_1)[\rslope[\min](\vec u) - \lslope[\min](\vec u)] > 0
  \end{align*}
  where the first equality uses that $\sum_{j \in \SS_i} \norm{x_j}^2 = \dn \big(\sum_{j \in \SS_i} \myd_{jj}\big) \delta_{\myD}(\SS_j) = \dn$ for all $j \in [\edim]$.
  It follows that $\vec e_1$ is a local maximum of $H$.
\end{proof}

\begin{thm} \label{app_sym:thm:complete_enumeration} 
  In Construction~\ref{constr:spectral-Fg-general}, $\{\pm \myZv_i \suchthat i
  \in [\edim]\}$ is a complete enumeration of the local maxima of
  $\fg$.
\end{thm}
\begin{proof}
  Let $\Lambda$ denote the set of local maxima of $\fg$.  That $\Lambda \subset 
  \{\pm  \myZv_i \suchthat i \in [\edim]\}$ is immediate from Lemma
  \ref{lem:app:lsym_no_extraneous_maxima}.  To see that
  $\Lambda \supset \{\pm  \myZv_i : i \in [\edim]\}$, we note that there is a
  natural mapping between $\Delta^{\Zpts-1}$ and a quadrant of $\sphere^{\edim - 1}$.
  
  The set $\{\pm  \myZv_i \suchthat i \in [\edim] \}$ gives an unknown, orthonormal basis of our space.  We may without loss of generality work in the coordinate system where
  $\vec e_1, \dotsc, \vec e_{\edim}$ coincide with $ \myZv_1, \dotsc,  \myZv_{\edim}$.
  Let $Q_1 = \sphere^{\edim-1} \cap [0, \infty)^{\edim - 1}$ give the first quadrant of
  the unit sphere.
  By the symmetries of the problem, it suffices to show that 
  $\{\vec e_1, \dotsc, \vec e_{\edim}\}$ are maxima of $\fg$.  However, the map
  $\psi : Q_1 \rightarrow \Delta^{\edim-1}$ defined by $(\psi({\vec u}))_i = u_i^2$
  is a homeomorphism.  Defining $H : \Delta^{\edim - 1} \rightarrow \R$ by
  $H(\vec t) = \fg(\psi^{-1}(\vec t))$, then $\vec t \in \Delta^{\edim - 1}$ is a local maximum 
  of $H$ if and only if $\psi^{-1}(\vec t)$ is a local maximum of $\fg$ relative to $Q_1$.
  
  Note that $H(\vec t) = \frac 1 \dn \sum_{i=1}^{\edim} \sum_{j \in \calS_i} g(\sqrt{t_i \norm{\vec x_{j}}^2})$.
  As $y \mapsto g(\sqrt{y})$ is convex, it follows by 
  Lemma~\ref{lem:app:convex_inner_sufficiency} that 
  $\{\vec e_i\}_{i=1}^{\edim}$ are local maxima of $H$.  Hence, using the symmetries of $\fg$,
  $\{\pm  \myZv_i \suchthat i \in [\edim]\} \supset \Lambda$.
\end{proof}

With Theorem~\ref{app_sym:thm:complete_enumeration} in hand, it is now
straight forward to generalize Theorem~\ref{thm:complete_enumeration}
to demonstrate that the spectral embedding arising from any of the
graph Laplacians is compatible with the proposed BEF function
maximization framework for clustering within the embedded space.
\begin{thm}\label{thm:spectralembBEF-complete-enumeration-general}
  Suppose that $G$ is a graph consisting of $\dn$ vertices and $\edim$ connected components with indices in the sets $\SS_1, \dotsc, \SS_\edim$.
  Let $\mathcal L$ be a (well defined) graph Laplacian chosen among $L$, $L_{\rw}$, or $L_{\sym}$ constructed from $G$.
  If $X \in \R^{\dn \times \edim}$ is such that its columns $x_{\vecdot i}$ form an orthogonal subspace of $\NN(\LL)$ scaled such that $\norm{x_{\vecdot i}} = \sqrt \dn$, then there exists an orthonormal basis $\myZv_1, \dotsc, \myZv_\edim$ of $\R^\edim$ such that
  \begin{enumerate}
  \item\label{thm:spectralembBEF-complete-enumeration-general-1} For
    each $j \in \SS_i$, $x_{j \vecdot}^T$ lies on the ray starting at
    the origin and going through $\myZv_i$.
  \item\label{thm:spectralembBEF-complete-enumeration-general-2} If we
    define $\fg: \sphere^{\edim - 1} \rightarrow \R$ by $\fg(\vec u) =
    \frac 1 \dn \sum_{i=1}^\dn g(\ipCanonical{\vec u}{x_{i \vecdot}})$
    from a contrast $g : [0, \infty) \rightarrow \R$ satisfying that
    $t \mapsto g(\sqrt t)$ is strictly convex, then the directions
    $\{\pm \myZv_i \suchthat i \in [\edim]\}$ provide a complete
    enumeration of the local maxima of $\fg$ on $\sphere^{\edim - 1}$.
  \end{enumerate}
\end{thm}
\begin{proof}
  Part~\ref{thm:spectralembBEF-complete-enumeration-general-1} follows
  from the combination of Propositions~\ref{prop:discrete-simplex},
  \ref{app_sym:prop:ray_structure}, and~\ref{prop:rw-embedding}.
  Part~\ref{thm:spectralembBEF-complete-enumeration-general-2} follows
  from Theorem~\ref{app_sym:thm:complete_enumeration} along with the
  observation that Construction~\ref{constr:spectral-Fg-general}
  captures the given $\fg$ irregardless of which of the 3 graph
  Laplacians is used to construct $\fg$ (see
  Construction~\ref{constr:spectral-Fg-general} and the surrounding
  discussion).
\end{proof}

\ifnum\version=1
\appendix
\section{Facts About Convex Functions}
\fi
\ifnum\version=2
\appendix[Facts About Convex Functions]
\fi
\label{app:facts-convex}
In this section, intervals can be open, half open, or closed.

There is a large literature studying the properties of convex functions.
As strict convexity is considered more special than convexity, results
are typically stated in terms of convex functions.  The following characterization
of strict convexity is a version of Proposition 1.1.4 of \citep{hiriart1996convex} for strictly convex
functions, and can be proven in a similar fashion.  
\begin{lem} \label{lem:convexity_increasing}
  For an interval $I$, let $f:I \rightarrow \R$ be a strictly convex function.  Then,
  fixing any $x_0 \in I$, the slope function defined by
  $\SCHAR(x) := \frac{f(x) - f(x_0)}{x - x_0}$
  is strictly increasing on $I \setminus \{ x_0 \}$.
\end{lem}

The following result is largely a consequence of Lemma \ref{lem:convexity_increasing}.
\begin{lem} \label{lem:convex-diff-quotient-increasing}
  Let $I$ be an interval and let $f: I \rightarrow \R$ be a convex function.  Suppose
  that $(a, b) \subset I$ and $(c, d) \subset I$ are such that $a \leq c$ and 
  $b \leq d$ with at least one of the inequalities being strict.  Then,
  \begin{equation*}
  \frac{f(b) - f(a)}{b - a} < \frac{f(d) - f(c)}{d - c}
  \end{equation*}
  \begin{proof}
    If $c=a$, then $\frac{f(d)-f(a)}{d-a} = \frac{f(d) - f(c)}{d-c}$ trivially.  Otherwise,
    $a < c$, and by Lemma \ref{lem:convexity_increasing}, we have that
    $\frac{f(d) - f(a)}{d-a} < \frac{f(d) - f(c)}{d-c}$
    By similar reasoning, $\frac{f(b) - f(a)}{b-a} \leq \frac{f(d) - f(a)}{d-a}$ (with equality if and only if $d = b$).
    As by assumption, $a=b$ and $c=d$ cannot both hold, it follows that
    $\frac{f(b) - f(a)}{b-a} \leq \frac{f(d) - f(a)}{d-a} \leq \frac{f(d) - f(c)}{d-c}$ with at least one of the inequalities being strict.
  \end{proof}
\end{lem}

The following result comes from Remark 4.2.2 of \ifnatbib \citet{hiriart1996convex}. \else Hiriart-Urruty and Lemar\'echal \cite{hiriart1996convex}. \fi
\begin{lem} \label{lem:continuity_derivs}
  Given an interval $I$ and a function $f: I \rightarrow \R$, then the
  left derivative $\partial_{-} f$ is left-continuous and the right derivative
  $\partial_{+} f$ is right-continuous respectively whenever they are defined
  (that is, finite).
\end{lem} 


\section*{Acknowledgements}
This work was supported by NSF grants IIS 1117707, CCF 1350870, and CCF 1422830.

\ifnatbib
\bibliographystyle{abbrvnat}
\else
\bibliographystyle{IEEEtran}
\fi
\bibliography{biblio}

\ifnum\version=2
\begin{IEEEbiography}[{\includegraphics[width=1in,height=1.25in,clip,keepaspectratio]{photos/jimmy}}]{James Voss}
  is a PhD student in the Department of Computer Science and
  Engineering at the Ohio State University.
  Voss received an MS in computer science in 2014 from Ohio State.
  His research interests include machine learning theory, independent
  component analysis, and unsupervised learning techniques.
\end{IEEEbiography}

\begin{IEEEbiography}{Mikhail Belkin}
  Mikhail Belkin is an Associate Professor in the Department of
  Computer Science and Engineering and in the Department of Statistics
  at the Ohio State University.
  He received his PhD in 2003 from the Department of
  Mathematics at the University of Chicago.
  His research interests are in the theory and applications of machine
  learning and data analysis.
  His work includes the Laplacian Eigenmaps algorithm, which brought
  ideas from classical differential geometry to data analysis and
  Polynomial Learning of Distribution Families, which used
  semi-algebraic geometry for provable learning of Gaussian mixture
  distributions. 
\end{IEEEbiography}

\begin{IEEEbiography}[{\includegraphics[width=1in,height=1.25in,clip,keepaspectratio]{photos/luis}}]{Luis Rademacher}
  graduated with a PhD in mathematics from the Massachusetts Institute of
  Technology in 2007 under the guidance of Santosh Vempala.
  He spent two years as a Postdoctoral Fellow in the College of Computing at
  the Georgia Institute of Technology.
  He joined the Computer Science and Engineering Department at The Ohio State
  University in 2009.
  His interests lie around computational learning theory and random structures
  and algorithms.
\end{IEEEbiography}
\vfill
\fi

\end{document}